%% file: main.tex
\documentclass{article}

\usepackage[main,final]{neurips_2026}

\usepackage[utf8]{inputenc} 
\usepackage[T1]{fontenc}    
\usepackage{hyperref}       
\usepackage{url}            
\usepackage{booktabs}       
\usepackage{amsfonts}       
\usepackage{nicefrac}       
\usepackage{microtype}      
\usepackage{xcolor}         

\usepackage{microtype}
\usepackage{graphicx}
\usepackage{subfigure}
\usepackage{multirow}
\usepackage{pgfplots}
\usepackage{placeins}
\usepackage{wrapfig}
\usepgfplotslibrary{groupplots}
\pgfplotsset{compat=1.17}

\usepackage{hyperref}

\input{macros}
\input{shortcuts}

\usepackage{amsmath}
\usepackage{amssymb}
\usepackage{mathtools}
\usepackage{amsthm}
\usepackage{array}

\usepackage[capitalize,noabbrev]{cleveref}

\theoremstyle{plain}
\newtheorem{theorem}{Theorem}[section]
\newtheorem{proposition}[theorem]{Proposition}
\newtheorem{lemma}[theorem]{Lemma}

\theoremstyle{definition}

\theoremstyle{remark}

\usepackage{tikz}
\usetikzlibrary{shapes.geometric, arrows.meta, positioning, calc, fit, backgrounds}

\title{Probabilistic Circuits for Irregular Multivariate Time Series Forecasting}
\author{%
  Christian Klötergens \\
  ISMLL \& VWFS DARC\\
  University of Hildesheim\\
  Hildesheim, Germany \\
  \texttt{kloetergens@ismll.de} \\
  \And{}
  Lars Schmidt-Thieme \\
    ISMLL \& VWFS DARC\\
  University of Hildesheim\\
  Hildesheim, Germany \\
  \And{}
  Vijaya Krishna Yalavarthi \\
  ISMLL \& VWFS DARC\\
  University of Hildesheim\\
  Hildesheim, Germany \\
}
\begin{document}
\addtocontents{toc}{\protect\setcounter{tocdepth}{-1}}
\maketitle






\begin{abstract}
    Joint probabilistic modeling is essential for forecasting irregular multivariate time series (IMTS) to accurately quantify uncertainty.
    Existing approaches often struggle to balance model expressivity with consistent marginalization, frequently leading to unreliable or contradictory forecasts. 
    To address this, we propose \textbf{\model{}}, a novel architecture for probabilistic IMTS forecasting based on sum-product networks (SPNs).\@ 
    Our model is flexible in capturing intricate dependencies between time series channels while structurally guaranteeing valid joint distributions. 
    Experiments on five real-world datasets demonstrate that \model{} achieves superior joint and marginal density estimation compared to state-of-the-art baselines.
\end{abstract}
\input{content/intro.tex}
\input{content/related_work.tex}
\input{content/problem.tex}
\input{content/architecture.tex}
\input{content/results.tex}

\begin{ack}
The authors gratefully acknowledge the computing time granted by the Resource Allocation Board and provided on the supercomputer Grete at NHR-Nord@Göttingen as part of the NHR infrastructure. The calculations for this research were conducted with computing resources under the project nim00026.
\end{ack}

{
\small
\bibliographystyle{plainnat}
\bibliography{references}
}

\clearpage
\appendix
\addtocontents{toc}{\protect\setcounter{tocdepth}{2}}
\renewcommand{\contentsname}{Appendix Contents}
\tableofcontents
\clearpage
\input{appendix/notation.tex}
\clearpage
\input{appendix/proof.tex}
\clearpage
\input{appendix/expressivity.tex}
\clearpage
\input{appendix/model_details.tex}

\clearpage
\input{appendix/bifurcation_details.tex}

\clearpage
\input{appendix/datasets.tex}
\clearpage
\input{appendix/hyperparameters.tex}
\clearpage
\input{appendix/abl_details.tex}
\clearpage
\input{appendix/additional_metrics.tex}
\clearpage
\input{appendix/additional_baselines.tex}

\clearpage
\input{appendix/channel_ordering.tex}
\clearpage
\input{appendix/query_dependent_weights.tex}
\clearpage
\input{appendix/k_sensitivity.tex}

\clearpage
\input{appendix/scaling.tex}
\clearpage
\input{appendix/broader_impacts.tex}
\clearpage
\input{checklist.tex}

%
%


\end{document}

%% file: macros.tex
\usepackage{pifont}%

\usepackage{xargs}
\usepackage[colorinlistoftodos,textsize=normalsize]{todonotes} %
\newcommandx{\todoc}[2][1=]{{\todo[linecolor=orange,backgroundcolor=orange!25,bordercolor=orange,#1]{
			TODO:\@ #2}}}
\newcommandx{\unsure}[2][1=]{{\todo[linecolor=yellow,backgroundcolor=yellow!25,bordercolor=yellow,#1]{
			UNSURE:\@ #2}}}
\newcommandx{\change}[2][1=]{{\todo[linecolor=blue,backgroundcolor=blue!25,bordercolor=blue,#1]{
			CHANGE:\@ #2}}}
\newcommandx{\info}[2][1=]{{\todo[linecolor=green,backgroundcolor=green!25,bordercolor=green,#1]{
			INFO:\@ #2}}}
\newcommandx{\improvement}[2][1=]{{\todo[linecolor=violet,backgroundcolor=violet!25,bordercolor=violet,#1]{
			IMPROVEMENT:\@ #2}}}
\newcommandx{\thiswillnotshow}[2][1=]{{\todo[disable,#1]{THIS WILL NOT SHOW:\@
			#2}}}

\usepackage{bm}

\newcommand{\model}{CircuITS}

\usepackage{amsmath,amssymb}

\usetikzlibrary{colorbrewer} \usetikzlibrary{positioning}
\usetikzlibrary{shapes.geometric} \usetikzlibrary{arrows}
\usetikzlibrary{arrows.meta} \usetikzlibrary{patterns,decorations.pathreplacing}
\usetikzlibrary{fit} \usetikzlibrary{backgrounds}
\usetikzlibrary{shapes,backgrounds,calc} \usetikzlibrary{shadows}
\usetikzlibrary{patterns} \definecolor{mygray}{RGB}{190,190,190}
\newcommand{\node}{\mathsf{N}}

\newcommand{\x}{\mathbf{x}}

\newlength{\minnodesize}
\newlength{\nodethickness}
\newlength{\nodedist}
\newlength{\prodnodedist}
\newlength{\leafnodedist}
\definecolor{tab10green}{HTML}{2CA02C}
\definecolor{tab10blue}{HTML}{1f77b4}
\definecolor{tab10red}{HTML}{d62728}

\newcommand{\wroot}{\mathbf{w}^{\text{root}}}

\newcommand{\pro}{\text{Lin}}
\newcommand{\softmax}{\text{softmax}}

\definecolor{profiti_color}{HTML}{356208}  
\definecolor{moses_color}{HTML}{DD7514}    
\definecolor{circuits_color}{HTML}{E916D7} 

%% file: shortcuts.tex
\newcommand{\X}{\mathcal{X}}
\newcommand{\Q}{\mathcal{Q}}
\newcommand{\y}{\mathbf{y}}
\newcommand{\qry}{^\text{qry}}
\newcommand{\myoplus}{%
	\mathbin{\ooalign{$\bigcirc$\cr\hss\raise0ex\hbox{+}\hss\cr}}%
}
\newcommand{\myotimes}{%
	\mathbin{\ooalign{$\bigcirc$\cr\hss\raise0.1ex\hbox{\small $\times$}\hss\cr}}%
}
\newcommand{\myo}{%
	\mathbin{\ooalign{$\bigcirc$\cr\hss\raise0.1ex\hbox{}\hss\cr}}%
}
\renewcommand{\c}{c}

\newcommand{\R}{\mathbb{R}}
\newcommand{\mlp}{\text{MLP}}

%% file: content/intro.tex
\section{Introduction}
Forecasting irregular multivariate time series (IMTS) means predicting the value of a variable (a \emph{channel}) at a given time from sparse, irregularly sampled observations. Because real-world IMTS are generated by stochastic processes, point predictions are rarely sufficient: a useful forecast must quantify uncertainty and capture dependencies between channels, which requires predicting a \emph{joint} distribution over a flexible number of query variables.
A critical limitation of the state-of-the-art model in probabilistic IMTS forecasting (ProFITi~\citep{Yalavarthi2025.Probabilistic}) is the lack of marginalization consistency~\citep{Yalavarthi2025.Reliable}.
This flaw leads to contradictory predictions.
For example, a weather model might predict that snow is unlikely at $10^\circ C$, yet revise that probability the moment we also ask about humidity.
In a well-behaved model, adding or removing variables should not alter the predicted distribution over the remaining ones.
When a model is inconsistent, its answers are artifacts of the question rather than the underlying distribution.
\input{figures/bifurcation.tex}

Recently MOSES~\citep{Yalavarthi2025.Reliable} was introduced to be a marginalization consistent, but expressive model for probabilistic IMTS forecasting. 
MOSES, however, lags behind ProFITi in estimating accurate conditional joint densities.
While theoretically capable of approximating any distribution, MOSES struggles with the exponential scaling required to represent independent multi-modal processes. As shown in \Cref{fig:bifurcation}, accurately modeling a simple bifurcation across four independent channels requires a mixture of $2^C=16$ components to represent all joint modes ($C$ representing the number of channels). MOSES with 10 components fails due to insufficient capacity; however, even with 16 components MOSES fails to capture the distinct branches despite having sufficient components in principle. The optimization landscape for such high-dimensional mixtures becomes difficult to navigate, leading to poor local minima that fail to represent the ground truth distribution. While one could trivially solve this specific example by manually enforcing independence between channels, real-world data contains a complex, unknown mix of dependencies that cannot be hard-coded upfront.
Hence, there seems to be a trade-off between 
accurate modeling and marginalization consistency.
This is not only a dilemma for practitioners but also unintuitive as marginalization consistency should introduce the correct inductive bias into the model.

Consequently, we require an architecture that avoids this trade-off. 
The ideal model should be capable of representing complex joint dependencies where they exist, while efficiently factoring into independent components where they do not. 
This motivates our proposal of \textbf{Circu}its for \textbf{I}rregular Multivariate \textbf{T}ime \textbf{S}eries (\textbf{\model{}}), a model that leverages the hierarchical nature of Sum-Product Networks (SPNs)~\citep{Poon2011.Sumproduct}, a class of probabilistic circuits~\citep{Choi2020Probabilistic}.
SPNs are computation graphs that represent a joint density as a hierarchy of weighted mixtures (sum nodes) and conditional-independence factorizations (product nodes), with tractable distributions at the leaves. This structure makes both density evaluation and marginalization linear in the size of the graph.
By recursively stacking sum and product nodes, \model{} flexibly models the joint distribution over IMTS forecasting queries. 
This allows it to capture intricate correlations and independent factors alike, while strictly guaranteeing marginalization consistency. 
As a result, \model{} trivially resolves the independent bifurcation task shown in \cref{fig:bifurcation}. 
Furthermore, our experiments on five real-world datasets, featuring varying lookback windows and forecasting horizons, demonstrate that \model{} not only outperforms MOSES on every dataset in terms of joint density but also surpasses ProFITi and establishes a new state-of-the-art.
Our contributions can be summarized as follows\footnote{Code is available here: \url{https://github.com/kloetergensc/CircuITS-for-IMTS}}:
\begin{enumerate}
    \item We propose \model{}, a novel architecture to model joint densities of IMTS queries. Unlike previous models, it structurally guarantees marginalization consistency, ensuring valid and non-contradictory predictions for any subset of variables.
    \item We introduce an encoder that handles irregular data to generate encodings of forecasting queries and the weights of the SPN.
    \item We demonstrate through extensive experiments on five datasets that \model{} achieves superior joint and marginal density estimation compared to its baselines. 
\end{enumerate}

%% file: figures/bifurcation.tex
\begin{figure*}
    \centering
    \includegraphics[width=1.35in]{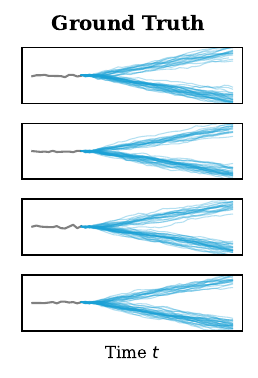}
    \includegraphics[width=1.35in]{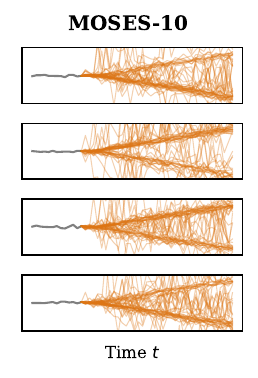}
    \includegraphics[width=1.35in]{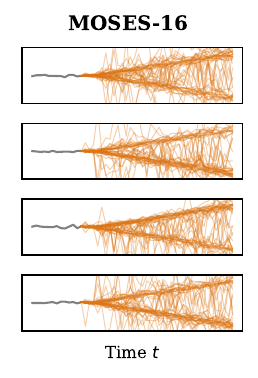}
    \includegraphics[width=1.35in]{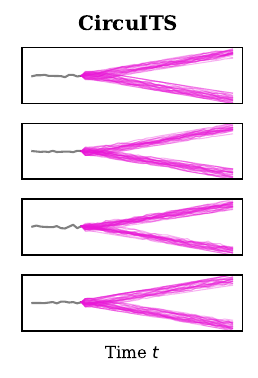}
    \caption{Demonstrating MOSES' inability to learn a multivariate bifurcation with 4 independent channels. MOSES-$K$, refers to MOSES with $K$-many mixture components. The plots of \model{} are created by using 2 circuit components.}\label{fig:bifurcation}
\end{figure*}

%% file: content/related_work.tex
\section{Related Work}

Recently published approaches that model IMTS often focus on point forecasting using set/graph based models~\citep{Yalavarthi2024.GraFITi,Klotergens2024.Functional,Zhang2024.Irregular,Luo2025.HiPatch,Li2025.HyperIMTS}.
While effective for deterministic tasks, many downstream applications require quantifying uncertainty.

Early probabilistic approaches for IMTS extrapolation are restricted to predicting the \textit{univariate marginal distributions} (i.e., the distribution of a single channel at a single time step) by modeling a stochastic process with  a Neural ODE~\citep{Chen2018.Neural}. Models such as NeuralFlows \citep{Bilos2021.Neural}, GRU-ODE-Bayes \citep{DeBrouwer2019.GRUODEBayes}, and Continuous Recurrent Units (CRU)~\citep{Schirmer2022.Modeling} fall into this category. 

Modeling the joint distribution allows for capturing dependencies, but it introduces the challenge of \textit{marginalization consistency}. This property is not merely a mathematical constraint; it represents the correct inductive bias for any model approximating a stochastic process~\citep{Oksendal2003.Stochastic, Yalavarthi2025.Reliable}. Models that violate this, such as ProFITi~\citep{Yalavarthi2025.Probabilistic}, lack this inductive bias, leading to contradictory predictions where the forecast for a subset of variables disagrees with the forecast derived from the joint distribution~\citep{Murphy2022.Probabilistic}. Such inconsistencies render models unreliable for critical decision-making in high-stakes domains like healthcare.
While marginalization consistency is essential, current consistent approaches struggle to combine this inductive bias with sufficient expressivity for high-dimensional joint distributions.
Gaussian Process Regression (GPR)~\citep{Durichen2015.Multitask} satisfies consistency but is severely limited by its reliance on Gaussian assumptions.
MOSES (Mixtures of Separable Flows)~\citep{Yalavarthi2025.Probabilistic} attempts to bridge this gap by parameterizing a mixture of multivariate Gaussians combined with separable normalizing flows applied to the univariate marginals.
This architecture guarantees consistency and achieves high expressivity for univariate predictions.
However, its ability to model complex \textit{joint} dependencies is bottlenecked by the mixture model itself.
While Gaussian mixtures are theoretically universal approximators, the amount of data required to accurately estimate an increasing number of components, grows exponentially with the number of variables~\citep{Anderson2014.More}.

TACTiS~\citep{Drouin2022.TACTiS} and TACTiS-2~\citep{Ashok2023.TACTiS2}, employ Transformers to parameterize attentional copulas, achieving state-of-the-art performance on multivariate probabilistic forecasting. However, these models are primarily designed for regularly sampled or fully observed settings and their implementations are not readily adaptable to the sparse irregular setting.
Both TACTiS and TACTiS-2 do not provide structural guarantees to be valid copulas, but implement learning procedures so that the final model parameters hopefully result in valid copulas. 


%% file: content/problem.tex
\section{Probabilistic IMTS forecasting}~\label{sec:problem_formulation}
We represent an irregular multivariate time series (IMTS) as a sequence of observations $\mathcal{X}$. Unlike regular time series represented as a matrix, $\mathcal{X}$ is a sequence of $L$ triplets:
$    \mathcal{X} = {\Bigl((t_n, c_n, y_n) \Bigr)}_{n=1}^L, $

with $ (t_n, c_n, y_n) \in \mathbb{R} \times \{1, \dots, C\} \times \mathbb{R}$,
where $t_n$ denotes the continuous timestamp, $c_n$ is the channel index, and $y_n$ is the observed scalar value. The observations are typically sparse and irregularly sampled across time and channels.

\paragraph{Forecasting Query.}
A forecasting task is defined by a query $\mathcal{Q}$, which is a sequence of $N$ target pairs specifying where predictions are required:
$\mathcal{Q} = {\Bigl( (t_n^{\text{qry}}, c_n^{\text{qry}}) \Bigr)}_{n=1}^{N}$,
with ($t_n^{\text{qry}}, c_n^{\text{qry}}) \in \mathbb{R} \times \{1, \dots, C\}$. The goal is to predict the values $\mathbf{y} = {(y_1, \dots, y_{N})}^\top$ corresponding to the query points in $\mathcal{Q}$, conditioned on the observed history $\mathcal{X}$. In general, the order of the sequences $\X$  and $(\Q, \y)$ is of no importance, hence we perceive them often  as sets.

\paragraph{Probabilistic Objective.}
Given the stochastic nature of the underlying processes generating IMTS data, point predictions are often insufficient \citep{Murphy2022.Probabilistic}. Therefore, we aim to model the full conditional joint probability density function over the query targets:
\begin{equation}
    {p}(\mathbf{y} \mid \mathcal{Q}, \mathcal{X}): \mathbb{R}^N \to \mathbb{R}_{\geq 0}
\end{equation}
This mapping must be valid for any query size $N$ and any configuration of query points $\mathcal{Q}$.
To ensure the model serves as a reliable approximation of the underlying stochastic process, we impose the following strict requirements:

\textbf{1. Flexibility and Joint Modelling.} The model must accept variable-length observation sequences $\mathcal{X}$ and queries $\mathcal{Q}$. Crucially, it must model the dependencies between target variables by outputting a joint distribution, rather than factorized independent marginals.

\textbf{2. Marginalization Consistency.} A valid stochastic process must satisfy the Kolmogorov extension theorem~\citep{Oksendal2003.Stochastic}. 
This implies that the predicted marginal distribution of a query subset must match the integrated joint distribution. 
Let $\y_{S}$ denote a subset of variables that is dropped from $\y$. Then the model fulfills \textit{Marginalization Consistency} if:
\begin{equation}
	\int p(\y \mid \Q, \X) \, \text{d}\y_{S} = p(\y_{-S} \mid \Q_{-S}, \X)
\end{equation}
This property ensures that the prediction for a specific variable does not change arbitrarily simply because we decided to also query another variable~\citep{Yalavarthi2025.Reliable}.

%% file: content/architecture.tex
\section{\model{}}
\model{} models the joint distribution over IMTS forecasting queries in two stages. First, for each channel, we model the marginal distribution over its query points as a mixture. Second, we combine these per-channel mixtures into a single joint distribution over all channels using a recursive sum-product network (SPN~\citep{Poon2011.Sumproduct}, \Cref{sec:rec_spn}). \Cref{fig:architecture} illustrates the full pipeline.

\begin{figure}[!htp]
\centering
\resizebox{\textwidth}{!}{%
\input{tikz_images/architecture.tex}%
}
\caption{\model{}' architecture illustrated for 3 channels A, B, C and $K=2$ mixture component. The channel-wise component distributions are modeled using separable Flows and Gaussian Copulas (GC) and then combined by the recursive sum-product network (SPN). }\label{fig:architecture}
\end{figure}

\input{content/encoder.tex}
\input{content/leaves.tex}

\subsection{Recursive Sum-Product Aggregation}~\label{sec:rec_spn}
To capture dependencies across channels, we propose a \textit{Recursive Sum-Product Network (SPN)} that combines the per-channel mixtures from \Cref{sec:leaves} into a single joint distribution. An SPN is a computation graph over three kinds of nodes: \emph{leaves}, which represent tractable density functions over a subset of variables; \emph{product nodes}, which multiply children with disjoint scopes and thereby encode conditional independence; and \emph{sum nodes}, which take a convex combination of children with the same scope and thereby encode mixtures. Provided every product is over disjoint scopes (decomposability) and every sum is over identical scopes (smoothness), the network represents a valid joint density that can be evaluated and marginalized in time linear in its size~\citep{Choi2020Probabilistic}.

In \model{}, each channel $c$ contributes $K$ leaves ${\{p_{c, k}\}}_{k=1}^K$ --- the channel-wise component distributions defined in \Cref{sec:leaves}. We build the joint density over all channels by processing them in sequence, alternating product nodes (to combine the previously aggregated channels with a new channel) and sum nodes (to mix the resulting components back to a fixed size $K$). Let $p^{(c)}(\mathbf{y}_{1:c})$ denote the joint density over the first $c$ channels, represented by a vector of $K$ component densities $\mathbf{\Phi}^{(c)}(\mathbf{y}_{1:c}) = {[\phi^{(c)}_1, \dots, \phi^{(c)}_K]}^\top$. The aggregation proceeds recursively:

\paragraph{Initialization ($c=1$):}
For the first channel, the joint state is simply the vector of its leaf densities:
\begin{equation}
    \phi^{(1)}_k(\mathbf{y}_1) = p_{1, k}(y_1), \quad \forall k \in \{1, \dots, K\}
\end{equation}
\paragraph{Recursive Step ($c > 1$):}
To incorporate channel $c$, we instantiate one product node for every pair $(i, j)$ of a previous component $\phi^{(c-1)}_i$ and a new leaf $p_{c, j}$. This yields $K^2$ product terms:
\begin{equation}\label{eq:recurs}
    \psi_{m}(\mathbf{y}_{1:c}) = \phi^{(c-1)}_i(\mathbf{y}_{1:c-1}) \cdot p_{c, j}(y_c)
\end{equation}
where $m$ indexes the pair $(i, j)$ in the flattened space of size $K^2$. Decomposability is satisfied because $\phi^{(c-1)}_i$ has scope $\mathbf{y}_{1:c-1}$ and $p_{c, j}$ has scope $y_c$. Intuitively, each product term represents a scenario in which the previous channels and the new one can be conditionally independent given a pair of latent components.
A naive composition would let the number of components grow as $K^c$, since each new channel multiplies the size of the joint state by $K$. We therefore follow each product layer with $K$ sum nodes that project the $K^2$ products back onto a basis of size $K$. Each sum node $k$ at level $c$ takes a convex combination of the $K^2$ products and shares the same scope $\mathbf{y}_{1:c}$ as its children, so smoothness holds:
\begin{equation}\label{eq:sum_layers}
    \phi^{(c)}_k(\mathbf{y}_{1:c}) = \sum_{m=1}^{K^2} w^{(c)}_{m, k} \cdot \psi_{m}(\mathbf{y}_{1:c}), \quad \text{with} \sum_{m=1}^{K^2} w^{(c)}_{m, k} = 1
\end{equation}
Here, $\mathbf{W}^{(c)} \in \mathbb{R}^{K^2 \times K}$ is a context-dependent weight matrix that captures the correlations between the previous joint state and the new channel. Intuitively: since all product nodes are valid distributions over the same random variables, any convex sum of them is also a valid distribution. Therefore, every sum node represents a valid density too.

\paragraph{Full Joint Distribution and Training Objective.}
After processing all $C$ channels, the final joint likelihood is given by a weighted sum of the final $K$ components using the root weights $\wroot$:
\begin{equation}
    p(\mathbf{y} \mid \mathcal{Q}, \mathcal{X}) = \sum_{k=1}^K w^{\text{root}}_k \cdot \phi^{(C)}_k(\mathbf{y})
\end{equation}
Because every leaf, product, and sum node in the construction admits an exact tractable density, this joint log-likelihood can be evaluated end-to-end on a single forward pass through the recursion in~\eqref{eq:recurs}--\eqref{eq:sum_layers} and is fully differentiable in the encoder, the leaf parameters, and the SPN weights. We therefore train \model{} by directly minimizing the normalized joint Negative Log Likelihood (njNLL), which divides the joint log-likelihood by the number of query targets $N$ to make the objective comparable across queries of varying size:
\begin{equation}\label{eq:njnll}
    \text{njNLL}(\mathbf{y} \mid \mathcal{Q}, \mathcal{X}) = \frac{-\log p(\mathbf{y} \mid \mathcal{Q}, \mathcal{X})}{N}.
\end{equation}
For numerical stability, all computations are carried out in the log-domain: products become sums and the sum layer is implemented via the Log-Sum-Exp (LSE) function. The recursive update from~\eqref{eq:recurs} thus becomes:
    $\log \phi^{(c)}_k = \mathrm{LSE}\left( \log \mathbf{\Psi} + \log \mathbf{W}^{(c)}_{\cdot, k} \right)$,
where $\log \mathbf{\Psi}$ collects the log-probabilities of the pairwise products.
\paragraph{SPN Weights.}
The sum-layer weights are conditioned on the observations $\mathcal{X}$ through the channel embeddings $\tilde{\mathbf{H}}$ produced by the encoder (\Cref{sec:encoder}). Starting from a set of learnable channel-wise queries $\mathbf{E}^\text{cw} \in \R^{C \times D}$, we attend over $\tilde{\mathbf{H}}$ via multi-head attention:
\begin{equation}
    \mathbf{Z} = \mathrm{MHA}(\mathbf{E}^\text{cw}, \tilde{\mathbf{H}}, \tilde{\mathbf{H}}), \quad \mathbf{Z} \in \R^{C \times D}.
\end{equation}
Each row $\mathbf{Z}_c$ then yields the weights for the $c$-th sum layer (and, for $c=1$, the root weights):
\begin{align}
    \wroot = \softmax \big( \pro^\textsc{root}(\mathbf{Z}_1) \big), \quad
    \mathbf{W}^{(c)} = \softmax \big( \pro^\textsc{w}(\mathbf{Z}_c )\big)\label{eq:WCfromZ}
\end{align}
where $\pro^\textsc{root}: D \to K$ and $\pro^\textsc{w}: D \to K^2 \times K$ are linear projections, and the softmax in~\eqref{eq:WCfromZ} is taken over the $K^2$ dimension so that each column forms a valid mixture (cf.~\eqref{eq:sum_layers}).

\subsection{Sampling}\label{sec:sampling}
Samples from the joint distribution are generated by ancestral sampling: starting at the root and traversing down to the leaves. Because of our recursive construction, this top-down traversal corresponds to iterating through channels from $C$ down to $1$ (see \Cref{app:sampling_spn} for details).
At each leaf, given the active component $k$, we draw $z \sim \mathcal{N}(0, \mathbf{R}^{c,k})$, map it to the uniform domain via $u = \Phi(z)$, and numerically invert the DSF by bisection to obtain $\hat{y}_i = F^{-1}(u_i; \theta_i^k)$~\citep{Drouin2022.TACTiS}.

\subsection{Computational Complexity}\label{sec:complexity}
For Gaussian-based joint models, the bottleneck is the Cholesky decomposition over all $N$ query points, which scales as $\mathcal{O}(N^3)$. \model{} avoids this by handling inter-channel dependencies through the SPN and restricting matrix operations to the per-channel leaves. Assuming each channel has $N/C$ query points, the total cost drops to $\mathcal{O}(C \cdot {(N/C)}^3) = \mathcal{O}(N^3/C^2)$ — a saving of factor $C^2$ over the dense Gaussian baseline~\citep{Yalavarthi2025.Reliable}.
The price is cubic scaling in the number of mixture components $K$: each recursive step projects $K^2$ products back to $K$ components via $\mathbf{W}^{(c)} \in \mathbb{R}^{K^2 \times K}$, costing $\mathcal{O}(K^3)$ per channel. We therefore keep $K \leq 4$ in our experiments.
Although the SPN aggregation is sequential by construction, it can be parallelized across channels using the Hillis-Steele scan~\citep{Hillis1986.Data} (see \Cref{app:parallel_circuit}).
We analyze which joint densities the recursive projection onto $K$ sum nodes can represent in \Cref{app:expressivity}.

\subsection{Relation to the SPN Literature}\label{sec:spn_lit}
\model{} combines two ideas that the SPN literature has studied separately.
Conditional SPNs~\citep{Shao2020.Conditional} make the sum weights functions of a set of conditioning variables, which is also true for \model{}. However, in a conditional SPN the modeled variables are fixed, so every remaining variable is available to condition on. In IMTS forecasting, which random variables need to be predicted (i.e.\ the query $\mathcal{Q}$) is itself part of the model input, and the sum weights must not depend on it to preserve marginalization consistency (\Cref{app:proof}).
Dynamic SPNs~\citep{Melibari2016.Dynamic} handle variable-length sequence data by unrolling a template SPN over time steps. We instead unroll a template over the $C$ channels and absorb the variable number of query points into the leaves, whose scope $\mathbf{y}^{(c)}$ has instance-dependent size $N_c$.

%% file: tikz_images/architecture.tex
\begin{tikzpicture}[
    >={Latex[length=1.6mm, width=1.6mm]},
    rb/.style = {draw, rounded corners=2pt, align=center, line width=0.6pt},
    enc/.style  = {rb, fill=orange!25, minimum width=0.75cm, minimum height=5.4cm, font=\bfseries},
    qbox/.style = {rb, fill=blue!12, minimum width=0.95cm, minimum height=0.55cm,  inner sep=2pt},
    leaf/.style = {rb, fill=green!18, minimum width=2cm, minimum height=0.6cm, inner sep=2pt},
    stagebox/.style = {draw, rounded corners=3pt, dashed, line width=0.5pt, color=black!55, inner sep=4pt},
    phi/.style  = {rb, fill=cyan!18, minimum width=0.85cm, minimum height=0.55cm, inner sep=1pt},
    plus/.style = {circle, draw, line width=0.7pt, minimum size=0.55cm, inner sep=0pt, fill=red!12, font=\bfseries},
    times/.style= {circle, draw, line width=0.7pt, minimum size=0.55cm, inner sep=0pt, fill=yellow!28, font=\bfseries},
    root/.style = {circle, draw, line width=1pt, minimum size=0.7cm, inner sep=0pt, fill=red!22, font=\bfseries},
    outbox/.style  = {rb, fill=red!18, minimum width=2.0cm, minimum height=0.7cm, line width=0.9pt, font=\bfseries, align=center},
    plotframe/.style = {draw, rounded corners=2pt, fill=gray!5, draw=black!40, line width=0.5pt,
                        minimum width=2.4cm, minimum height=2.2cm, inner sep=0pt},
    desc/.style = {font=\itshape, fill=gray!18, rounded corners=2pt, inner sep=2pt, inner ysep=1.5pt, draw=black!35, line width=0.4pt},
    sechdr/.style = {font=\bfseries\large, fill=gray!18, rounded corners=2pt, inner sep=3pt, inner ysep=2pt, draw=black!35, line width=0.4pt},
    ar/.style   = {->, semithick, color=black!75},
    e/.style    = {-, thin, color=black!55},
    we/.style   = {-, thick, color=black!75},
]
    \node[enc] (enc) {\rotatebox{90}{Encoder (\ref{sec:encoder})}};

    \node[plotframe, above left=4mm and 6mm of enc.west, anchor=south east] (xplot) {};
    \begin{scope}[shift={(xplot.south west)}]
        \draw[->, gray!70, line width=0.4pt] (0.15,0.2) -- (2.25,0.2) node[right, yshift=-1pt] {$t$};
        \draw[->, gray!70, line width=0.4pt] (0.15,0.2) -- (0.15,2.05);
        \draw[blue!75!black, thick] plot[smooth, tension=0.6] coordinates {(0.25,1.10) (0.7,1.55) (1.25,1.25) (1.8,1.70) (2.15,1.45)};
        \foreach \x/\y in {0.7/1.55, 1.25/1.25, 1.8/1.70} \fill[blue!75!black] (\x,\y) circle (1.4pt);
        \draw[red!75!black, thick] plot[smooth, tension=0.6] coordinates {(0.25,0.85) (0.65,1.00) (1.15,0.70) (1.65,1.05) (2.15,0.85)};
        \foreach \x/\y in {0.65/1.00, 1.15/0.70, 1.65/1.05} \fill[red!75!black] (\x,\y) circle (1.4pt);
        \draw[green!55!black, thick] plot[smooth, tension=0.6] coordinates {(0.25,0.35) (0.75,0.50) (1.35,0.25) (1.85,0.45) (2.15,0.30)};
        \foreach \x/\y in {0.75/0.50, 1.85/0.45} \fill[green!55!black] (\x,\y) circle (1.4pt);
    \end{scope}
    \node[desc, above=0.5mm of xplot] {Observations $\mathcal{X}$};

    \node[qbox, draw=blue!75!black, very thick,
          below left=4mm and 1cm of enc.west, anchor=north east] (q1) {$\mathcal{Q}_{A}$};
    \node[qbox, draw=red!75!black,  very thick, below=1.5mm of q1] (q2) {$\mathcal{Q}_{B}$};
    \node[qbox, draw=green!55!black,very thick, below=1.5mm of q2] (q3) {$\mathcal{Q}_{C}$};
    \node[desc, align=center, left=1mm of q2] {\rotatebox{90}{$\mathcal{Q}$ per channel}};

    \draw[ar] (xplot.east) -- (xplot.east -| enc.west);
    \draw[ar, color=blue!75!black]   (q1.east) -- (q1.east -| enc.west);
    \draw[ar, color=red!75!black]    (q2.east) -- (q2.east -| enc.west);
    \draw[ar, color=green!55!black]  (q3.east) -- (q3.east -| enc.west);

    \node[leaf, right=6mm of enc.east, yshift=25mm] (l1a) {$p_{1}(A|\Q_A,\X)$};
    \node[leaf, below=1mm of l1a] (l1b) {$p_{2}(A| \Q_A,\X)$};
    \node[leaf, below=7mm of l1b] (l2a) {$p_{1}(B|\Q_B,\X)$};
    \node[leaf, below=1mm of l2a] (l2b) {$p_{2}(B|\Q_B,\X)$};
    \node[leaf, below=7mm of l2b] (l3a) {$p_{1}(C|\Q_C,\X)$};
    \node[leaf, below=1mm of l3a] (l3b) {$p_{2}(C|\Q_C,\X)$};

    \node[sechdr, above=5mm of l1a] (dsfandgc) {Flows $+$ GC (\ref{sec:leaves})};
    \node[sechdr, right=30mm of dsfandgc] {Recursive SPN (\ref{sec:rec_spn})};

    \foreach \L in {l1a, l1b, l2a, l2b, l3a, l3b}
        \draw[ar] (enc.east |- \L) -- (\L.west);

    \node[times, right=8mm of l1b, yshift=2mm]   (x2a) {$\boldsymbol{\times}$};
    \node[times, below=1.5mm of x2a]               (x2b) {$\boldsymbol{\times}$};
    \node[times, below=1.5mm of x2b]               (x2c) {$\boldsymbol{\times}$};
    \node[times, below=1.5mm of x2c]               (x2d) {$\boldsymbol{\times}$};
    \foreach \pp in {l1a, l1b}
        \foreach \xx in {x2a, x2b, x2c, x2d}
            \draw[e] (\pp.east) -- (\xx.west);
    \foreach \ll in {l2a, l2b}
        \foreach \xx in {x2a, x2b, x2c, x2d}
            \draw[e] (\ll.east) -- (\xx.west);

    \node[plus, right=4mm of x2b] (s2a) {$\boldsymbol{+}$};
    \node[plus, right=4mm of x2c]              (s2b) {$\boldsymbol{+}$};
    \foreach \xx in {x2a, x2b, x2c, x2d}
        \foreach \ss in {s2a, s2b}
            \draw[we] (\xx.east) -- (\ss.west);

    \node[phi, right=2mm of s2a] (phi2a) {$\phi^{(2)}_{1}$};
    \node[phi, right=2mm of s2b] (phi2b) {$\phi^{(2)}_{2}$};
    \draw[ar] (s2a) -- (phi2a);
    \draw[ar] (s2b) -- (phi2b);

    \node[times, right=5mm of phi2b, yshift=-3mm]   (x3a) {$\boldsymbol{\times}$};
    \node[times, below=1.5mm of x3a]               (x3b) {$\boldsymbol{\times}$};
    \node[times, below=1.5mm of x3b]               (x3c) {$\boldsymbol{\times}$};
    \node[times, below=1.5mm of x3c]               (x3d) {$\boldsymbol{\times}$};
    \foreach \pp in {phi2a, phi2b}
        \foreach \xx in {x3a, x3b, x3c, x3d}
            \draw[e] (\pp.east) -- (\xx.west);
    \foreach \ll in {l3a, l3b}
        \foreach \xx in {x3a, x3b, x3c, x3d}
            \draw[e] (\ll.east) -- (\xx.west);

    \node[plus, right=4mm of x3b] (s3a) {$\boldsymbol{+}$};
    \node[plus, right=4mm of x3c]              (s3b) {$\boldsymbol{+}$};
    \foreach \xx in {x3a, x3b, x3c, x3d}
        \foreach \ss in {s3a, s3b}
            \draw[we] (\xx.east) -- (\ss.west);

    \node[phi, right=2mm of s3a] (phi3a) {$\phi^{(3)}_{1}$};
    \node[phi, right=2mm of s3b] (phi3b) {$\phi^{(3)}_{2}$};
    \draw[ar] (s3a) -- (phi3a);
    \draw[ar] (s3b) -- (phi3b);

    \node[plus, right=2mm of phi3a,yshift=-3.5mm] (root) {$\boldsymbol{+}$};
    \draw[we] (phi3a.east) -- (root.west);
    \draw[we] (phi3b.east) -- (root.west);

    \node[outbox, above=10mm of root] (outnode) {$p(\mathbf{y}\mid\mathcal{Q},\mathcal{X})$};
    \draw[ar] (root) -- (outnode);



\end{tikzpicture}

%% file: content/encoder.tex
\subsection{Encoder}\label{sec:encoder}
We aim to infer the conditional joint distributions ${p}(\mathbf{y} \mid \mathcal{Q}, \mathcal{X})$ of IMTS forecasting queries. 
Therefore, we need an encoder to extract the relevant information contained in $\X$ and $\Q$ into latent embeddings which can be used 
to parameterize the distributions. To guarantee that we compute likelihoods marginalization consistently, the elements of $\mathcal{Q}$ need to be encoded \emph{seperable}.

\paragraph{Observation Representation.} 
Each observation is a triplet consisting of a continuous timestamp $t_n \in \mathbb{R}$, a channel index $c_n \in \{1, \dots, C\}$, and a scalar value $y_n \in \mathbb{R}$. We project these discrete observations into a latent feature space of dimension $D$ ($\mathbf{H}^{\text{obs}} \in \mathbb{R}^{L \times D}$). The $n$-th observation embedding $\mathbf{h}_n \in \mathbb{R}^D$ is computed as
    $\mathbf{h}_n =  \left[ \phi^\textsc{time}(t_n) \, \| \, \mathbf{E}^{\text{chan}}_{c_n} \, \| \, y_n \right]$,
where $\mathbf{E}^{\text{chan}} \in \mathbb{R}^{C \times D}$ is a learnable channel embedding matrix, $\phi^\textsc{time}(\cdot)$ is a sinusoidal time embedding~\citep{Kazemi2019.Time2Vec}, and $[\cdot \| \cdot]$ denotes concatenation.
\paragraph{Intra-Channel Aggregation.} 
To aggregate an arbitrary number of observations per channel into a fixed-size representation, we employ Multi-Head Attention (MHA)~\citep{Vaswani2017.Attention}. We introduce a set of static learnable embeddings $\mathbf{E}^{\text{proto}} \in \mathbb{R}^{C \times D}$, where each row represents a channel prototype. These serve as queries in a masked Cross-Attention layer where we use the observation embeddings $\mathbf{H}^{\text{obs}}$ as keys and values:
\begin{equation}
    \mathbf{H}^{\text{chan}} = \mathrm{MHA}\left( \mathbf{E}^{\text{proto}}, \mathbf{H}^{\text{obs}}, \mathbf{H}^{\text{obs}}; \mathbf{M} \right)
\end{equation}
Here, $\mathbf{M}$ is a binary attention mask where $M_{c,n} = 1$ if observation $n$ belongs to channel $c$, and $0$ otherwise. 
This operation effectively \emph{groups} observations by channel, producing a matrix $\mathbf{H}^{\text{chan}} \in \mathbb{R}^{C \times D}$ that summarizes the observations of each channel.
\paragraph{Inter-Channel Interaction.} 
To capture dependencies between variables (e.g., correlations between channels), we update the channel representations using a Self-Attention layer:
\begin{equation}
     \tilde{\mathbf{H}} = \mathbf{H}^{\text{chan}} + \mathrm{MHA}(\mathbf{H}^{\text{chan}}, \mathbf{H}^{\text{chan}}, \mathbf{H}^{\text{chan}})
\end{equation}
The updated states $\tilde{\mathbf{H}} \in \mathbb{R}^{C \times D}$ serve as the context for the subsequent generation tasks. If a channel is completely unobserved, its representation in $\tilde{\mathbf{H}}$ is derived solely from its static query and interactions with other channels.
\paragraph{Encoding the Query.} 
For a target query at time $t^{\text{qry}}_n$ on channel $c_n$, we generate the context embedding $\mathbf{e}_n$ by retrieving the corresponding row $\tilde{\mathbf{h}}_{c_n}$ from the channels' hidden state and combining it with the query time:
    $\mathbf{e}_n = \pro^\textsc{qry}\left( \left[ \phi^\textsc{time}(t^{\text{qry}}_n) \, \| \, \tilde{\mathbf{h}}_{c_n} \right] \right)$.
For simplicity, we use a single linear layer ($\pro^\textsc{qry}$) to project the concatenation back to dimension $D$.
Note that $\tilde{\mathbf{h}}_{c_n}$ is computed before any query is introduced, so per-channel parameters depend only on $\mathcal{X}$ --- a property we will rely on later to establish marginalization consistency (\Cref{app:gc_proof}).
We provide a detailed architecture diagram of our encoder in \Cref{app:enc}.

%% file: content/leaves.tex
\subsection{Marginal Distributions of Channels (Leaf Distributions)}\label{sec:leaves}
For each channel $c$, we model $K$ \emph{candidate} distributions (leaves) over its query elements. \Cref{sec:rec_spn} then combines these per-channel candidates into the full joint via the SPN.\@ Intuitively, each candidate captures one plausible regime that the channel could follow. In the bifurcation task of \Cref{fig:bifurcation}, for instance, $K = 2$ candidates per channel suffice: one for the upward branch and one for the downward branch. These $K$ candidates take the role of mixture components in the resulting SPN, and we use both terms interchangeably below.

For a given channel $c$, let $\mathbf{y}^{(c)} = (y_1, \dots, y_{N_c})$ denote the sequence of target values (observations) and $\mathbf{E}^{(c)} = (\mathbf{e}_1, \dots, \mathbf{e}_{N_c})$ denote the corresponding context embeddings generated by the encoder, where $\mathbf{e}_i \in \R^D$. We split the embeddings into $K$ partitions $\mathbf{e}_i^{k} \in \R^{D^\prime}$, with $D = KD^\prime$, so that each candidate is parameterized from a separate slice of the encoder output.
\paragraph{Univariate Marginals.}
The Deep Sigmoidal Flow (DSF)~\citep{Huang2018.Neural} approximates the cumulative distribution function (CDF) via a composition of $L^\textsc{dsf}$ monotonic layers, denoted as $F(y_i; \theta_i^k) = f_{L^\textsc{dsf}} \circ \dots \circ f_1(y_i)$. The $l$-th layer operation is defined as:
\begin{equation}
    f_l(x) = \sigma(A_l x + b_l),
    \label{eq:dsf_layer}
\end{equation}
where $\sigma(\cdot)$ is the sigmoid activation function. The parameters $\theta_i^k$ include the weight matrices $A$ and bias vectors $b$ for all layers. To condition the flow on the input context, we use a two-layer MLP that maps the query embedding $e_i^k$ to the instance-specific flow parameters $\theta_i^k$.
To ensure the strict monotonicity required for a valid normalizing flow, all elements of $A$ are enforced to be strictly positive. The sigmoid activation in the last layer ensures the output lies in $(0, 1)$, effectively modeling the CDF of a continuously defined distribution. The density for the observation $y_i$ is then computed via the change of variables formula:
\begin{equation}
    \log p(y_i | e_i^k) = \log \frac{\partial F(y_i; \theta_i^k)}{\partial y_i}
    \label{eq:dsf_loglik}
\end{equation}
The derivative $\frac{\partial F}{\partial y_i}$ is obtained via a forward-pass application of the chain rule.
\paragraph{Gaussian Copula.}
The Gaussian copula is a statistical construct that models the dependence structure between multiple random variables by mapping their marginal CDF to a multivariate standard normal distribution. It assumes that the dependence between variables is captured entirely by a correlation matrix $\mathbf{R}$, allowing for the decoupling of marginal distributions from the joint dependency structure.
The log-likelihood of a Gaussian copula is defined as:
\begin{equation}
\mathcal{L}^\textsc{GC}(\mathbf{u},R) = -\frac{n}{2} \log |R| - \frac{1}{2} \sum_{i=1}^{n} \mathbf{z}_i^\top (R^{-1} - I) \mathbf{z}_i,
\end{equation}
where, $\mathbf{z}_{i} = \Phi^{-1}(\mathbf{u}_{i})$, represents the inverse CDF of a standard normal distribution~\citep{Nelsen2006.Introduction}.
To apply Gaussian copulas for our channel-wise marginal distributions we compute a correlation matrix $\mathbf{R}^{c,k} \in \mathbb{R}^{N_c \times N_c}$ for each channel $c$ and mixture component $k$.
First, we project the embeddings into a feature space of dimension $D^\textsc{corr}$ in which we aim to model correlations:
\begin{equation}\label{eq:v}
    \mathbf{v}_i^{c,k} = \tanh(\text{MLP}_{\text{corr}}(\mathbf{e}^k_i)) \quad \in \mathbb{R}^{D^\textsc{corr}}
\end{equation}
The raw correlation score between query target $i$ and target $j$ is the scaled inner product of their features.
The $\tanh$ activation ensures that any $\mathbf{v}_i^\top \mathbf{v}_j$ is within $(-D^\textsc{corr},D^\textsc{corr})$. Therefore, instead of scaling by $1/\sqrt{D^\textsc{corr}}$ which is commonly used in scaled dot product attention, 
we divide by $D^\textsc{corr}$ to ensure $\mathbf{R}^{c,k} \in (-1,1)$.
To ensure a valid correlation structure, we obtain $\mathbf{R}^{c,k}$ by replacing the diagonal of $\dot{R}^{c,k}$ with ones:
\begin{equation}\label{eq:R}
    \mathbf{R}^{c,k} = \dot{R}^{c,k} + (\mathbf{I} - \text{diag}(\dot{R}^{c,k})), \quad \text{with }  \dot{R}_{ij}^{c,k} = \frac{{(\mathbf{v}^{c,k}_i)}^\top (\mathbf{v}^{c,k}_j)}{D^\textsc{corr}}
\end{equation}
This preserves positive-definiteness: $\dot{R}^{c,k}$ is a Gram matrix and thus positive semi-definite, and the correction $(\mathbf{I} - \text{diag}(\dot{R}^{c,k}))$ is a diagonal matrix with strictly positive entries since $\dot{R}_{ii}^{c,k} = \|\mathbf{v}_i^{c,k}\|^2 / D^\textsc{corr} \leq 1$ (with equality only when $\mathbf{v}_i^{c,k}$ is fully saturated by $\tanh$).
With $\mathbf{R}^{c,k}$ we can combine the univariate marginals modeled by DSF to the log-likelihood function of the channel-wise candidate distribution:
\begin{equation}
    \log p_{c,k} = \underbrace{\sum^{N_c}_{i=1} \log p(y_i | e_i^k)}_{\text{univariate marginals}} + \underbrace{\mathcal{L}^\textsc{GC}(F(y_i, \theta^k_i),\mathbf{R}^{c,k})}_{\text{copula density}}
\end{equation}
Since we aim for \model{} to be marginalization consistent, the channel-wise candidate distributions are designed accordingly: normalizing flows model the univariate marginals, while Gaussian copulas combine them into a joint distribution.
We provide a more formal argument in~\Cref{app:gc_proof}.

%% file: content/results.tex
\section{Experiments}~\label{sec:exp}
We evaluate \model{} in two settings. First, we use a synthetic bifurcation task to probe whether the model can recover distinct multi-modal channel dynamics. Second, we benchmark \model{} on five real-world IMTS datasets (one climate, three medical, one human motion) covering a range of observation and prediction horizons, comparing against state-of-the-art probabilistic IMTS forecasting baselines and ablating the main components of \model{}.
\subsection{Independent Multi-channel Bifurcation}
We validate \model{}' capabilities on learning meaningful SPNs on a synthetic bifurcation task (\Cref{fig:bifurcation}).
The synthetic dataset consists of 10,000 simulated IMTS with four channels.
Each IMTS channel simulates a trajectory that starts as stationary Gaussian noise and bifurcates into a biased random walk (Brownian motion with drift) shortly after the final observation.
MOSES fails to capture the distinct modes even with 16 components due to optimization difficulties in high-dimensional mixture spaces. In contrast, \model{} solves the task with 2 components per channel, successfully isolating the upward and downward trends. We report the corresponding njNLLs in \Cref{app:bifurcation}.
\input{tables/main_table.tex}
\subsection{Benchmark on Real-World Datasets}
We evaluate \model{} on five datasets: one climate dataset (USHCN), three medical datasets (Physionet-2012, MIMIC-III, MIMIC-IV), and one human motion dataset (Human Activity~\citep{Kaluza2010.Agentbased}). Our evaluation covers a broad range of varying observation and prediction horizons, ranging from observing 36 hours/months to predict the next 3 steps (36--3$^\star$), to longer-range forecasting scenarios (e.g., 12--36, where the models are tasked to predict the final 36 hours / months). For Human Activity, a 4\,s window takes the role of the 48 hours (\Cref{app:datasets}).
For a fair comparison, we adopt the preprocessing, splitting, and metric-computation protocol of~\citet{Yalavarthi2025.Reliable}. Each model is trained and evaluated five times across different train/validation/test splits, and every model is evaluated on the \emph{same} five splits so that runs are paired across methods. Since the protocol is identical, we reuse their reported numbers for shared baselines and settings wherever applicable.
Additional information about the datasets (Appx.~\ref{app:datasets}) and hyperparameters  (Appx.~\ref{app:hyper}) can be found in the appendix.
\paragraph{Metrics.} Besides njNLL (\Cref{eq:njnll}), we report the mNLL (\Cref{app:mnll}), the Energy Score and CRPS (\Cref{app:crps_es}), and the MSE of the predictive mean (\Cref{tab:mse}). CRPS and mNLL assess only the univariate marginals and are blind to the dependence between channels, njNLL and the Energy Score assess the full joint distribution, and the MSE only measures point accuracy. Of the two joint metrics, njNLL is the more discriminative one: the Energy Score has low power to detect a misspecified dependence structure, with likelihood-based scores being far more discriminative~\citep{Marcotte2023.Regions}. Moreover, njNLL can be optimized directly, whereas the Energy Score is not differentiable end-to-end in \model{}, since sampling inverts the DSF by bisection (\Cref{sec:sampling}). We therefore focus on njNLL and also tune hyperparameters on it.
\paragraph{Baselines.} In our main experiment we compare against state-of-the-art probabilistic forecasting models for joint density estimation for IMTS forecasting: ProFITi~\citep{Yalavarthi2025.Probabilistic} and MOSES.\@ While MOSES guarantees marginalization consistency ProFITi does not as outlined in previous work.~\citep{Yalavarthi2025.Reliable}. Additionally, we compare against NeuralFlows~\citep{Bilos2021.Neural}, that predicts independent univariate normal distributions. Since it does not model dependencies between variables it is automatically marginalization consistent.  

\paragraph{Results.} The primary goal of probabilistic forecasting is to accurately model the joint distribution over future queries. As shown in \Cref{tab:njnll}, \model{} achieves the lowest (best) njNLL in 18 of the 20 scenarios.
While we train each model on njNLL, we also evaluate it on the marginal negative log-likelihood (mNLL) (\Cref{tab:mnll} in \Cref{app:mnll}). When we compare the mNLL, \model{} outperforms the baselines on majority of scenarios. On most datasets ProFITi has a relatively high mNLL highlighting the importance of marginalization consistency.

On the sample-based scores (\Cref{tab:energy,tab:crps}), \model{} attains the best CRPS in 13 and the best or second-best CRPS in 18 of the 20 settings. On the Energy Score it is best in 10 and best or second best in 15 of 20 settings, ahead of the strongest baseline (mean rank 1.75 vs.\ 1.95 for MOSES). On the MSE of the predictive mean (36--3$^\star$) it is the most accurate model on three of the four datasets (\Cref{tab:mse}). A low njNLL does not automatically translate into low sample-based scores: njNLL rewards placing sharp density on the observed values, whereas CRPS and the Energy Score are distance-based and penalize occasional outlying samples. They therefore disagree with njNLL precisely where sharp predictions are attainable. On tasks with relatively low njNLL (USHCN, and 36--3$^\star$ on all datasets) no model wins consistently, and the best Energy Score is attained by three different baselines, among them NeuralFlows, which models the channels as independent univariate normals. In the remaining scenarios, \model{} is best on both scores in most cases. It fares better on CRPS than on the Energy Score, because CRPS averages over the query points, damping a single badly placed coordinate by $1/N$, whereas in the Euclidean norm of the Energy Score the same coordinate inflates the distance of every sample.
We show an experiment about the impact of channel order in Appx.~\ref{app:order} and evaluate \model{}' efficiency in Appx.~\ref{app:effi}.

\paragraph{Ablations}
In order to validate the effectivity of the Sum-Product Network we evaluate a model (w/o SPN) that combines the channel marginals within a single product without mixing multiple components. This ablation models channels as independent marginals and is less expressive.
Additionally, we also present ablations where we replace the DSF with normal distributions (w/o DSF) and drop the Gaussian Copula (w/o GC). Finally, we also replace our encoder with Tripletformer (TF)~\citep{Yalavarthi2023.Tripletformer}. We provide more details on the ablation set up in \Cref{app:details_abl}. The respective results are shown in \Cref{tab:abl_spn}. Every ablation is worse than the full model on every dataset: the components are complementary rather than redundant.
\emph{Which component matters most is dataset-dependent.} The largest degradation stems from the DSF leaves on USHCN ($+2.47$) and MIMIC-III ($+0.29$), from the encoder on PhysioNet ($+0.36$), and from the SPN on MIMIC-IV ($+0.31$). Rather than a global ranking, the ablation thus shows which property of a dataset each component addresses: flexible marginals are decisive where the channel marginals are far from Gaussian (USHCN) and merely helpful where they are not (MIMIC-IV, $+0.18$).
\emph{The two dependence mechanisms contribute a small but consistent amount.} Removing the SPN (inter-channel dependence) costs $+0.23$ to $+0.31$ njNLL, removing the Gaussian copula (intra-channel dependence) $+0.17$ to $+0.23$. These two components buy expressivity \emph{without} sacrificing marginalization consistency.
\emph{The SPN ablation also diagnoses leaf specialization.} w/o SPN is exactly the case $K=1$. If the $K$ leaves of a channel converged to the same distribution, every sum node would mix identical components and the SPN would have no effect; the consistent gap therefore shows that the leaves specialize.
\emph{The encoder mainly affects stability.} Tripletformer is nearly free on MIMIC-III ($+0.09$) and MIMIC-IV ($+0.03$) but costly on USHCN ($+1.05$) and PhysioNet ($+0.36$), with standard deviations of up to $\pm0.39$, whereas the full model stays within $\pm0.16$ on all datasets (similarly, w/o DSF: $\pm0.30$ on USHCN). The full model is thus both the most accurate and the most stable across splits.
\input{tables/abl_spn.tex}

\section{Limitations}
The sequential SPN construction makes the modeled joint density dependent on channel order. We use a fixed arbitrary ordering since searching all $C\!$ permutations is intractable, although the empirical impact is small relative to the gap between models (\Cref{app:order}). Because inter-channel dependencies are captured via discrete mixtures rather than continuous correlations, smooth structures such as linear dependencies between channels cannot be represented compactly and must be approximated by a finite number of components. The $\mathcal{O}(N^3/C^2)$ complexity in \Cref{sec:complexity} additionally assumes evenly distributed queries. On real-world IMTS where queries concentrate on few channels (see \Cref{app:effi}), the per-channel grouping amortizes poorly and \model{} can have longer epoch-times than its baselines (\Cref{tab:mimiciiiet}). More generally, the circuit structure itself is fixed: channels are always combined sequentially along one chain, and the structure is neither learned nor adapted to the data.

\section{Conclusion}
In this work, we introduced \model{}, a novel Sum-Product Network architecture for irregular multivariate time series that resolves the trade-off between expressivity and marginalization consistency. Our experiments demonstrate that \model{} achieves state-of-the-art joint density estimation on real-world datasets while structurally guaranteeing valid, non-contradictory forecasts.
For future work, we plan to explore deeper and less rigid circuit structures than the fixed sequential chain, e.g.\ a tree-structured aggregation that recursively splits the channels into groups, aggregates each group independently, and merges the resulting sub-circuits pairwise.

%% file: tables/main_table.tex
\begin{table*}
\centering
\caption{Comparison of njNLL for different forecasting splits. Lower values are better. Best results are in \textbf{bold}, second best in \textit{italics}. The $\star$ represents that the forecasting horizon includes the next $n$ steps, instead of $n$-many hours/months.}\label{tab:njnll}
\setlength{\tabcolsep}{2mm}
\small
\begin{tabular}{clrrrrr}
    \toprule
    & Model & \multicolumn{1}{c}{USHCN} & \multicolumn{1}{c}{PhysioNet'12} & \multicolumn{1}{c}{MIMIC-III} & \multicolumn{1}{c}{MIMIC-IV} & \multicolumn{1}{c}{Activity} \\
    \midrule
    \multirow{4}{*}{\rotatebox{90}{36--3$^\star$}}
    & NeuralFlows & 0.775 ± 0.15 & 0.496 ± 0.00 & 0.998 ± 0.18 & 0.689 ± 0.09 & 0.360 ± 0.16 \\
    & ProFITi & -3.226 ± 0.23 & \textbf{-0.647 ± 0.08} & \textit{-0.377 ± 0.03} & \textit{-1.777 ± 0.07} & \textit{-0.023 ± 0.03} \\
    & MOSES & \textit{-3.357 ± 0.18} & -0.491 ± 0.04 & -0.305 ± 0.03 & -1.668 ± 0.10 & 0.144 ± 0.05 \\
    & \model{} & \textbf{-3.789 ± 0.22} & \textit{-0.550 ± 0.01} & \textbf{-0.574 ± 0.08} & \textbf{-2.113 ± 0.04} & \textbf{-0.062 ± 0.02} \\
    \midrule
    \multirow{4}{*}{\rotatebox{90}{36--12}}
    & NeuralFlows & 1.344 ± 0.13 & 0.708 ± 0.05 & 1.426 ± 0.12 & 1.033 ± 0.04 & 0.715 ± 0.08 \\
    & ProFITi     & \textit{-3.355 ± 0.20} & \textbf{-0.768 ± 0.04} & \textit{0.095 ± 0.16} & -0.511 ± 0.43 & 0.186 ± 0.05 \\
    & MOSES       & -1.790 ± 0.03 & -0.315 ± 0.02 & 0.275 ± 0.04 & \textit{-0.517 ± 0.03} & \textit{0.072 ± 0.02} \\
    & \model{}    & \textbf{-3.755 ± 0.16} & \textit{-0.536 ± 0.04} & \textbf{-0.253 ± 0.01} & \textbf{-1.201 ± 0.04} & \textbf{-0.120 ± 0.02} \\
    \midrule
    \multirow{4}{*}{\rotatebox{90}{24--24}}
    & NeuralFlows & 1.436 ± 0.20 & 1.097 ± 0.04 & 1.457 ± 0.06 & 1.195 ± 0.03 & 1.060 ± 0.12 \\
    & ProFITi     & -1.513 ± 0.14 & \textit{-0.355 ± 0.24} & \textit{0.196 ± 0.40} & 0.402 ± 0.21 & 0.428 ± 0.04 \\
    & MOSES       & \textit{-1.550 ± 0.65} & -0.298 ± 0.03 & 0.198 ± 0.01 & \textit{-0.478 ± 0.01} & \textit{0.107 ± 0.03} \\
    & \model{}    & \textbf{-3.585 ± 0.40} & \textbf{-0.523 ± 0.01} & \textbf{-0.112 ± 0.05} & \textbf{-0.851 ± 0.10} & \textbf{-0.074 ± 0.02} \\
    \midrule
    \multirow{4}{*}{\rotatebox{90}{12--36}}
    & NeuralFlows & 1.456 ± 0.27 & 1.436 ± 0.19 & 1.525 ± 0.05 & 1.186 ± 0.06 & 1.169 ± 0.10 \\
    & ProFITi     & -1.284 ± 0.51 & \textit{-0.291 ± 0.42} & \textit{0.285 ± 0.35} & -0.272 ± 0.40 & 0.597 ± 0.03 \\
    & MOSES       & \textit{-1.519 ± 0.17} & -0.063 ± 0.05 & 0.362 ± 0.00 & \textit{-0.294 ± 0.03} & \textit{0.157 ± 0.03} \\
    & \model{}    & \textbf{-3.400 ± 0.38} & \textbf{-0.402 ± 0.05} & \textbf{-0.462 ± 0.07} & \textbf{-0.859 ± 0.00} & \textbf{-0.021 ± 0.03} \\
\bottomrule
\end{tabular}
\end{table*}

%% file: tables/abl_spn.tex
\begin{table}
\centering
\caption{Showing the njNLL of \model{} compared and with ablations on the 36--12 task.}\label{tab:abl_spn}
\small
\setlength{\tabcolsep}{3mm} 
\begin{tabular}{lrrrr}
    \toprule
     & \multicolumn{1}{c}{USHCN} & \multicolumn{1}{c}{PhysioNet} & \multicolumn{1}{c}{MIMIC-III} & \multicolumn{1}{c}{MIMIC-IV} \\
    \midrule
    \model{}  & \textbf{-3.755 ± 0.16} & \textbf{-0.536 ± 0.04}  & \textbf{-0.253 ± 0.01} & \textbf{-1.201 ± 0.04} \\
    \midrule
    w/o SPN & -3.527 ± 0.05 & {-0.299 ± 0.05} & {0.013 ± 0.02} & {-0.892 ± 0.00} \\
    w/o DSF & -1.283 ± 0.30 & {-0.253 ± 0.17} & {0.040 ± 0.03} & {-1.019 ± 0.05} \\
    w/o GC  & -3.588 ± 0.08 & -0.303 ± 0.01   & -0.082 ± 0.06  & -0.972 ± 0.02 \\
    w/ TF   & -2.702 ± 0.39 & -0.173 ± 0.19   & -0.167 ± 0.15  & -1.167 ± 0.01 \\
    \bottomrule
\end{tabular}
\end{table}

%% file: appendix/notation.tex
\section{Notation}\label{app:notation}
\begin{table}[h!]
\centering\footnotesize
\setlength{\tabcolsep}{4pt}
\renewcommand{\arraystretch}{1.0}
\caption{Notation used in this work. Lowercase bold denotes vectors ($\y$, $\mathbf{e}_n$), uppercase bold denotes matrices ($\mathbf{R}$, $\mathbf{W}$), calligraphic letters denote sets or sequences ($\X$, $\Q$), $[n]$ abbreviates $\{1, \dots, n\}$, and a superscript in parentheses indexes the recursion step of the circuit ($\phi^{(c)}$) rather than a power.}\label{tab:notation}
\begin{tabular}{@{}llp{0.565\textwidth}@{}}
\toprule
Symbol & Domain / Shape & Meaning \\
\midrule
\multicolumn{3}{@{}l}{\textit{Data, task, and target distribution (\Cref{sec:problem_formulation})}} \\
$C$, $L$, $N$ & $\mathbb{N}$ & Number of channels, observations in $\X$, and query points \\
$N_c$ & $\mathbb{N}$ & Number of query points in channel $c$; $\sum_{c} N_c = N$ \\
$\X$ & ${(\R \times [C] \times \R)}^L$ & History ${((t_n, c_n, y_n))}_{n=1}^L$ of (time, channel, value) triplets \\
$\Q$ & ${(\R \times [C])}^N$ & Query ${((t\qry_n, c\qry_n))}_{n=1}^N$ of (time, channel) pairs \\
$\mathcal{I}_c$ & $\subseteq [N]$ & Index set of the query points of channel $c$: $\{n \mid c\qry_n = c\}$ \\
$\y$, $\y_c$, $\y_{1:c}$ & --- & Targets aligned with $\Q$; of channel $c$; of channels $1, \dots, c$ \\
$\y_S, \y_{-S}$ & --- & Targets that are marginalized out / retained, $S \subseteq [N]$ \\
$p(\y \mid \Q, \X)$ & $\R^N \!\to\! \R_{\geq 0}$ & Conditional joint density modeled by \model{} \\
$sc(\cdot)$ & --- & Scope of a circuit node: the subset of targets it is a density over \\
\midrule
\multicolumn{3}{@{}l}{\textit{Indices and hyperparameters}} \\
$n$, $c$ & $[L]$/$[N]$, $[C]$ & Observation / query index; channel index (= recursion step) \\
$k$; $i, j$ & $[K]$ & Leaf index; component of state $\phi^{(c-1)}$ / of new leaf $p_{c,\cdot}$ \\
$m$ & $[K^2]$ & Flattened index of the product node formed by the pair $(i,j)$ \\
$K$ & $\mathbb{N}$ & Number of leaves per channel; $K \in \{2,3,4\}$ in our experiments \\
$D$, $D'$ & $\mathbb{N}$ & Embedding dimension; dimension of one component slice, $D = K D'$ \\
$D^\textsc{corr}$ & $\mathbb{N}$ & Dimension of the correlation features (Eq.~\eqref{eq:v}) \\
$L^\textsc{dsf}$ & $\mathbb{N}$ & Number of layers of the Deep Sigmoidal Flow (DSF), indexed by $l$ \\
\midrule
\multicolumn{3}{@{}l}{\textit{Encoder (\Cref{sec:encoder}); all quantities depend on $\X$ only, except $\mathbf{e}_n$, which also depends on query point $n$}} \\
$\phi^\textsc{time}(\cdot)$ & $\R \to \R^{D}$ & Sinusoidal (Time2Vec) embedding of a continuous timestamp \\
$\mathbf{E}^\text{chan}$, $\mathbf{E}^\text{proto}$, $\mathbf{E}^\text{cw}$ & $\R^{C \times D}$ & Learnable channel embeddings, prototypes, and circuit-weight queries \\
$\mathbf{H}^\text{obs}$ & $\R^{L \times D}$ & Embeddings of the $L$ observation triplets \\
$\mathbf{M}$ & ${\{0,1\}}^{C \times L}$ & Attention mask; $M_{c,n} = 1$ iff observation $n$ belongs to channel $c$ \\
$\mathbf{H}^\text{chan}$ & $\R^{C \times D}$ & Per-channel summary of the observations \\
$\tilde{\mathbf{H}}$, $\tilde{\mathbf{h}}_c$ & $\R^{C \times D}$, $\R^{D}$ & Channel states after inter-channel self-attention; row $c$ of $\tilde{\mathbf{H}}$ \\
$\mathbf{e}_n$, $\mathbf{e}^k_n$ & $\R^{D}$, $\R^{D'}$ & Embedding of query point $n$; its $k$-th slice (parameterizes leaf $k$) \\
$\pro^\textsc{qry}, \pro^\textsc{w}, \pro^\textsc{root}$ & --- & Linear projections producing $\mathbf{e}_n$, $\mathbf{W}^{(c)}$, and $\wroot$ \\
\midrule
\multicolumn{3}{@{}l}{\textit{Leaves (\Cref{sec:leaves}); each $p_{c,k}$ is a density over all $N_c$ query points of channel $c$}} \\
$p_{c,k}$ & $\R^{N_c} \!\to\! \R_{\geq 0}$ & $k$-th candidate density (leaf) of channel $c$; models one plausible regime \\
$F(\cdot\,; \theta^k_i)$ & $\R \to (0,1)$ & DSF approximating the CDF of query point $i$ under component $k$ \\
$\theta^k_i$ & --- & DSF parameters $\{A_l, b_l\}_l$, layers $f_l(x) = \sigma(A_l x + b_l)$, $A_l > 0$ \\
$u_i$, $\mathbf{z}_i$ & $(0,1)$, $\R^{N_c}$ & PIT $u_i = F(y_i; \theta^k_i)$; normal scores $\mathbf{z}_i = \Phi^{-1}(\mathbf{u}_i)$ \\
$\mathbf{v}^{c,k}_i$ & $\R^{D^\textsc{corr}}$ & Correlation features of query point $i$, $\tanh(\mlp_\text{corr}(\mathbf{e}^k_i))$ \\
$\dot{R}^{c,k}, \mathbf{R}^{c,k}$ & $\R^{N_c \times N_c}$ & Raw Gram matrix and the resulting correlation matrix (unit diagonal) \\
$\mathcal{L}^\textsc{GC}$ & $\R$ & Log-density of the Gaussian copula \\
\midrule
\multicolumn{3}{@{}l}{\textit{Recursive sum-product circuit and objective (\Cref{sec:rec_spn}); $(c)$ denotes the state after absorbing channels $1, \dots, c$}} \\
$\phi^{(c)}_k$ & $\R^{N_{1:c}} \!\to\! \R_{\geq 0}$ & $k$-th component density of the joint state over channels $1{:}c$ \\
$\mathbf{\Phi}^{(c)}$ & $K$ densities & Vector ${[\phi^{(c)}_1, \dots, \phi^{(c)}_K]}^\top$ of the current joint state \\
$\psi_m$ & --- & Product node $m \!=\! (i,j)$: $\phi^{(c-1)}_i \cdot p_{c,j}$ \\
$\mathbf{W}^{(c)}$, $\wroot$ & $\R^{K^2 \times K}$, $\R^{K}$ & Column-stochastic sum weights at step $c$; root weights \\
$\mathbf{Z}, \mathbf{Z}_c$ & $\R^{C \times D}$, $\R^{D}$ & Attention output from which all circuit weights are projected \\
$M^{(c)}$ & $\R^{K \times K}$ & Log-semiring transition matrix of the parallel scan (\Cref{app:parallel_circuit}) \\
$k_c$ & $[K]$ & Active component of channel $c$ in sampling (\Cref{app:sampling_spn}) \\
njNLL & $\R$ & Training objective $-\log p(\y \mid \Q, \X) / N$ \\
\bottomrule
\end{tabular}
\end{table}

%% file: appendix/proof.tex
\section{Proofs}\label{app:proof}
\subsection{Marginalization Consistency up to the Leaves of \model{}}
\begin{lemma}\label{l:lemma1}An SPN is a marginalization consistent estimator up to the leaf nodes:
        \begin{equation*}
		\int p(\y \mid \Q, \X) \text{d}\y_\c = p(\y_{-\c} \mid \Q_{-\c}, \X),
	\end{equation*}
    where $\y_{-\c} = \y \setminus \y_\c$, 
    if it fulfills:
    \begin{itemize}
        \item \textbf{Property 1:} Each leaf node in the SPN represents the distribution of variables corresponding to a single input channel
        \item \textbf{Property 2:} The marginal leaf distributions only depend on the query subset ($\mathcal{Q}_c$) that corresponds to the respective channel
        \item \textbf{Property 3:} The weights of the sum nodes do not depend on the query set $\mathcal{Q}$
    \end{itemize}
\end{lemma}

\begin{proof}[Proof of \Cref{l:lemma1}]
	
Let $n$ be a node in the SPN representing the distribution $p_n(\y_{sc(n)} \mid \Q, \X)$.
We want to prove that for any channel $y_c \subseteq sc(n)$:
\[
\int p_n(\y_{sc(n)} \mid \Q, \X) \, d\y_\c = p_n(\y_{sc(n) \setminus \c} \mid \Q_{-\c}, \X)
\]
\paragraph{If $n$ is a Leaf Node.}
If $n$ is a leaf node $L_c^{(i)}$, its scope is $\y_\c$ (Property 1). By the definition of a probability distribution, integrating over its entire domain yields $1$:
\[\int p_L^{(i)}(\y_\c \mid \Q_c, \X) \, d\y_\c = 1\]
This echoes the principle of SPNs that the marginalization of a node is by replacing its distribution with $1$. The resulting $1$ is the identity element for product nodes, effectively \emph{removing} the channel from the joint factorization.

If $sc(n) \cap \y_\c = \emptyset$, the leaf is independent of the marginalization variable, and is therefore not affected.

\paragraph{If $n$ is a Product Node.}
By decomposability, $y_\c$ exists in the scope of exactly one child, say child $k'$. 
\[
p_n(\y_{sc(n)}\mid \Q, \X) = p_{k'}(\y_{sc(k')} \mid \Q, \X) \cdot \prod_{k \neq k'} p_k(\y_{sc(k)} \mid \Q, \X)
\]
Integrating over $\y_\c$:
\begin{multline*}
\int \left( p_{k'}(\y_{\c(k')} \mid \Q, \X) \prod_{k \neq k'} p_k(\y_{sc(k)} \mid \Q, \X) \right) d\y_\c \\
= \left( \int p_{k'}(\y_{sc(k')} \mid \Q, \X) d\y_\c \right) \prod_{k \neq k'} p_k(\y_{sc(k)} \mid \Q, \X)
\end{multline*}
If $sc(k') = \y_\c$, the integral is $1$, effectively removing channel $c$ from the product.
\paragraph{If $n$ is a Sum Node.} 
By smoothness of the SPN~\citep{Choi2020Probabilistic}, all children share the same scope. By Property 3, weights $w_{n,k}$ are independent of $\y$ and $\Q$.
\[
\int \left( \sum_{k \in ch(n)} w_{n,k}(\X) \cdot p_k(\y_{sc(n)} \mid \dots) \right) d\y_\c = \sum_{k \in ch(n)} w_{n,k}(\X) \int p_k(\y_{sc(n)} \mid \dots) d\y_\c
\]
Each integral $\int p_k(\y_{sc(n)}) d\y_\c$ yields the marginal $p_k(\y_{sc(n)\setminus c})$. Since the weights $w_{n,k}(\X)$ remain unchanged, the node $n$ now represents a valid mixture of the marginalized distributions of its children.

Since the property holds for the leaves and is preserved by both product and sum operations, it holds for the root of the circuit.
By repeating this process for any number of channels $c \in \{1, \dots, C\}$, we prove that the SPN is marginalization consistent across all channel-level subsets.
\end{proof}

\input{appendix/gc_proof.tex}

\subsection{Proof of \model{}' Marginalization Consistency}
\model{} is a marginalization consistent model for conditional joint density estimation of irregular multivariate time series forecasting. Formally, for any subset of variables $\y_A \subset \y$:
	\begin{equation*}
		\int p(\y \mid \Q, \X) \, d\y_A = p(\y \setminus \y_A \mid \Q \setminus \Q_A, \X)
	\end{equation*}

\begin{proof}
	Let $S$ be an arbitrary set of indices, and let $S_c = S \cap \mathcal{I}_c$ denote the indices belonging to channel $c$, such that $S = \bigcup_{c=1}^C S_c$.
	
	From \Cref{l:lemma1}, we have established that the circuit structure is marginalization-consistent up to the leaf nodes. The integral operator distributes through the sum and product nodes such that the global integral is determined by the local integrals at the leaves:
	\[
	\int p(\y \mid \Q, \X) d\y_{S_c} = \text{SPN}\left( \dots, \int p_L^{(i)}(\y_\c \mid \Q_\c, \X) d\y_{S_\c}, \dots \right)
	\]
	where $\text{SPN}(\cdot)$ represents the original circuit topology with fixed weights $w_{n,k}(\X)$ (Property 3 from \Cref{l:lemma1}).
	
	By \Cref{prop:gc}, each leaf node is itself marginalization consistent over its subset $S_\c$:
	\[
	\int p_L^{(i)}(\y_\c \mid \Q_\c, \X) d\y_{S_c} = p_L^{(i)}(\y_{\c \setminus S_c} \mid \Q_{\c \setminus S_\c}, \X)
	\]
	Substituting this into the circuit, the leaf nodes are replaced by their corresponding marginal distributions over the remaining queries $\Q_{c\setminus S_c}$. By repeating this process for all $c$, the entire circuit represents the valid marginal distribution $p(\y_{-S} \mid \Q_{-S}, \X)$.
\end{proof}

\subsection{On the Necessity of Query-Independent Sum Weights}\label{app:necessity}
\Cref{l:lemma1} states Property~3 as a sufficient condition. Here we discuss to what extent it is also necessary.
Expanding the recursion, the joint density of \model{} is a mixture over the $K^C$ configurations of leaf components (see \Cref{app:expressivity}),
\begin{equation*}
    p(\y \mid \Q, \X) = \sum_{j_1, \dots, j_C} \lambda_{j_1 \dots j_C} \prod_{c=1}^{C} p_{c, j_c}(\y_c),
\end{equation*}
with effective weights $\lambda_{j_1 \dots j_C}$ obtained by multiplying and summing the sum weights along the SPN.\@ Since leaves of different channels have disjoint scopes, the products $\prod_c p_{c,j_c}$ are linearly independent whenever the $K$ leaves of each channel are. Now drop a single query point $n$ from a channel. The set of channels is unchanged and the leaves are marginalization consistent (\Cref{prop:gc}), so the marginalized joint and the prediction for $\Q_{-n}$ are mixtures over the same product functions, with weights $\lambda(\Q)$ and $\lambda(\Q_{-n})$, respectively. By linear independence, consistency then forces $\lambda(\Q) = \lambda(\Q_{-n})$ exactly. Iterating this argument, the \emph{effective} weights must be independent of the query.

What is not necessary is that each individual sum weight $w^{(c)}_{m,k}$ is query-independent. The map from sum weights to effective weights is many-to-one, so individual weights could in principle vary with $\Q$ while the effective weights $\lambda(\Q)$ stay fixed. This is why we state Property~3 as a sufficient rather than a necessary condition. Such a variation, however, leaves every predicted density for every query unchanged: it is merely a reparameterization of the same circuit and therefore adds no expressivity. Moreover, guaranteeing that $\lambda(\Q)$ remains constant while the individual weights vary with $\Q$ would be very difficult to enforce architecturally. We quantify the inconsistency that arises from query-dependent sum weights empirically in \Cref{app:consistency_exp}.

%% file: appendix/gc_proof.tex
\subsection{\model{}' Leaves are Marginalization Consistent}~\label{app:gc_proof}
\begin{proposition}[Marginalization Consistency of \model{}' Leaf Distributions]\label{prop:gc}
Let $\mathcal{Q}_c \subseteq \mathcal{Q}$ be the subset of forecasting queries belonging to a specific channel $c$, and let $\mathbf{y}_c$ be the vector of target variables corresponding to $\mathcal{Q}_c$.

Consider an arbitrary query index $n$ within this channel's set $\mathcal{Q}_c$. Let $\mathcal{Q}_{c \setminus n}$ denote the query set with the $n$-th query point removed, and $\mathbf{y}_{c \setminus n}$ denote the target vector without $y_n$. The leaf distribution for channel $c$ is marginalization consistent if:
\begin{equation*}
    \int p(\mathbf{y}_c \mid \mathcal{Q}_c, \mathcal{X}) \, dy_n = p(\mathbf{y}_{c \setminus n} \mid \mathcal{Q}_{c \setminus n}, \mathcal{X})
\end{equation*}
By induction, satisfying this condition implies consistency for any subset of variables within the channel.
\end{proposition}

\begin{proof}
The leaf distribution for a channel $c$ and mixture component $k$ is defined as a Gaussian Copula combined with univariate Deep Sigmoidal Flows (DSF)~\citep{Nelsen2006.Introduction}. It is a standard property of the Gaussian Copula that marginalizing out a single variable $y_n$ yields a Gaussian Copula over the remaining variables $\mathbf{y}_{c \setminus n}$, parameterized by the correlation matrix $R^{c,k}$ with the $n$-th row and column removed.

Therefore, validity relies on \textbf{parameter stability}: the model parameters for the remaining queries $\mathcal{Q}_{c \setminus n}$ must not change when the query point $q_n$ is removed. Specifically, for any two remaining indices $i, j \neq n$, the correlation entry $R_{ij}$ must remain identical.

In \model{}, the correlation matrix entry for channel $c$ and component $k$ is computed as:
\begin{equation*}
    R_{ij} = \frac{\mathbf{v}_i^\top \mathbf{v}_j}{D^\textsc{corr}} + \delta_{ij}(1 - \dots),
\end{equation*}
where $\mathbf{v}_i = \tanh(\text{MLP}(e_i))$. The context embedding $e_i$ is derived via $e_i = \pro^\textsc{qry}([\phi^\textsc{time}(t^{\text{qry}}_i) \,\|\, \tilde{h}_{c}])$ (see \Cref{eq:v,eq:R}).

Crucially, the channel context $\tilde{h}_{c}$ is computed solely from the observation history $\mathcal{X}$ via the set-based encoder and is independent of the query set $\mathcal{Q}_c$. Thus, $e_i$ and $\mathbf{v}_i$ depend strictly on the specific query point $q_i=(t_i^\text{qry}, c)$ and $\mathcal{X}$. Since the calculation of $R_{ij}$ for $i,j \neq n$ does not involve the query point $q_n$, removing $q_n$ simply corresponds to removing the $n$-th row and column from $R^{c,k}$. This matches the theoretical marginalization requirement of the Gaussian Copula.
\end{proof}

%% file: appendix/expressivity.tex
\section{Expressivity of the Recursive Projection}\label{app:expressivity}
In every recursive step (\Cref{sec:rec_spn}), the $K^2$ product nodes are projected back onto $K$ sum nodes. Here we characterize precisely what this projection restricts.

\paragraph{What the projection does.} Unrolling the recursion in~\eqref{eq:recurs}--\eqref{eq:sum_layers} gives the closed form
\begin{equation}\label{eq:closed_form}
    p(\y \mid \Q, \X) = \sum_{k_1, \dots, k_C} A_{k_1 \dots k_C} \prod_{c=1}^{C} p_{c, k_c}(\y_c),
    \qquad
    A_{k_1 \dots k_C} = \mathbf{e}_{k_1}^\top \, \mathbf{G}^{(2)}[k_2] \cdots \mathbf{G}^{(C)}[k_C] \, \wroot,
\end{equation}
where $\mathbf{e}_{k_1}$ is the $k_1$-th unit vector and the per-channel factor $\mathbf{G}^{(c)}[j] \in \R^{K \times K}$ is the matrix whose entry in row $i$ and column $r$ equals the entry $W^{(c)}_{(i,j), r}$ of the sum-layer weights.
Hence, \model{} \emph{is} a mixture over all $K^C$ combinations of channel leaves. No product term is dropped, and every product term keeps a strictly positive weight, since all weights are obtained by a softmax (\Cref{eq:WCfromZ}). What the projection restricts is the \emph{weight tensor} $A$, which is forced into a nonnegative tensor-train (matrix-product-state) factorization of rank $K$. Equivalently, this is the hidden Markov model over the channel order described in \Cref{app:parallel_circuit}, with the $K$ sum nodes as latent states, the channel leaves as emissions, and $\mathbf{W}^{(c)}$ as the joint transition-emission matrix, so that the recursion is exactly its forward algorithm.

\paragraph{Intuition.} The $K$-dimensional message that channels $1, \dots, c$ pass to channels $c{+}1, \dots, C$ is $\mathbf{\Phi}^{(c)}$. Conditional on which of the $K$ sum nodes is active at level $c$, the channels up to $c$ and the channels after $c$ are independent. Each sum node thus acts as a reusable \emph{regime} summarizing the channels seen so far, and $\mathbf{W}^{(c)}$ specifies how leaf $j$ of a new channel pairs with regime $i$ to form regime $k$. The projection is what buys tractable evaluation and marginalization in $\mathcal{O}(CK^3)$ and keeps the network decomposable and smooth.

\paragraph{What is lost.} Any information about channels $1, \dots, c$ beyond \emph{which of the $K$ regimes is active} cannot influence the remaining channels. Formally, let $A^{(c)} \in \R^{K^c \times K^{C-c}}$ denote the matricization of $A$ that groups $(k_1, \dots, k_c)$ into rows and $(k_{c+1}, \dots, k_C)$ into columns. By~\eqref{eq:closed_form}, $A^{(c)} = LR$ with nonnegative factors $L \in \R_{\geq 0}^{K^c \times K}$ and $R \in \R_{\geq 0}^{K \times K^{C-c}}$, hence the nonnegative rank of $A^{(c)}$ is at most $K$, whereas an unconstrained weight tensor admits a nonnegative rank of up to $\min(K^c, K^{C-c})$. This yields the following statements:
\begin{itemize}
    \item For $C \leq 3$, the projection loses nothing: every weight tensor $A$ is exactly representable.
    \item For $C \geq 4$, it is a \emph{strict} restriction, first binding at the cut $c = 2$. Exactly representing an arbitrary $A$ would require the number of sum nodes per layer, i.e.\ the inner dimension of the factorization, to grow to $K^{\lfloor C/2 \rfloor}$.
\end{itemize}

\paragraph{Per-instance restriction.} Since the weights $\mathbf{W}^{(c)}$ and $\wroot$ are predicted per instance from $\X$ (\Cref{eq:WCfromZ}), the rank-$K$ constraint binds each \emph{predicted} joint density individually, not the model class as a whole: different instances (e.g.\ different patients) realize different rank-$K$ weight tensors.

\paragraph{Does the restriction matter in practice?} Increasing $K$ raises the tensor-train rank together with the number of leaves per channel and would therefore relax the restriction. Our sensitivity analysis (\Cref{app:k_sensitivity}) nevertheless finds $K = 2$ or $K = 3$ to be optimal, outperforming $K = 4$. This suggests that the rank constraint is not the binding limitation on the evaluated datasets.

%% file: appendix/model_details.tex
\section{Model Details}\label{app:model_details}
This section provides additional details on \model{}: a diagram of the encoder, the sampling procedure of the SPN, and the parallelized computation of the circuit.
\input{appendix/enc_fig.tex}
\input{appendix/circuit_sampling.tex}
\input{appendix/parallelizing_circuits.tex}

%% file: appendix/enc_fig.tex
\subsection{Encoder Architecture Diagram}\label{app:enc}
\input{tikz_images/enc.tex}

%% file: tikz_images/enc.tex
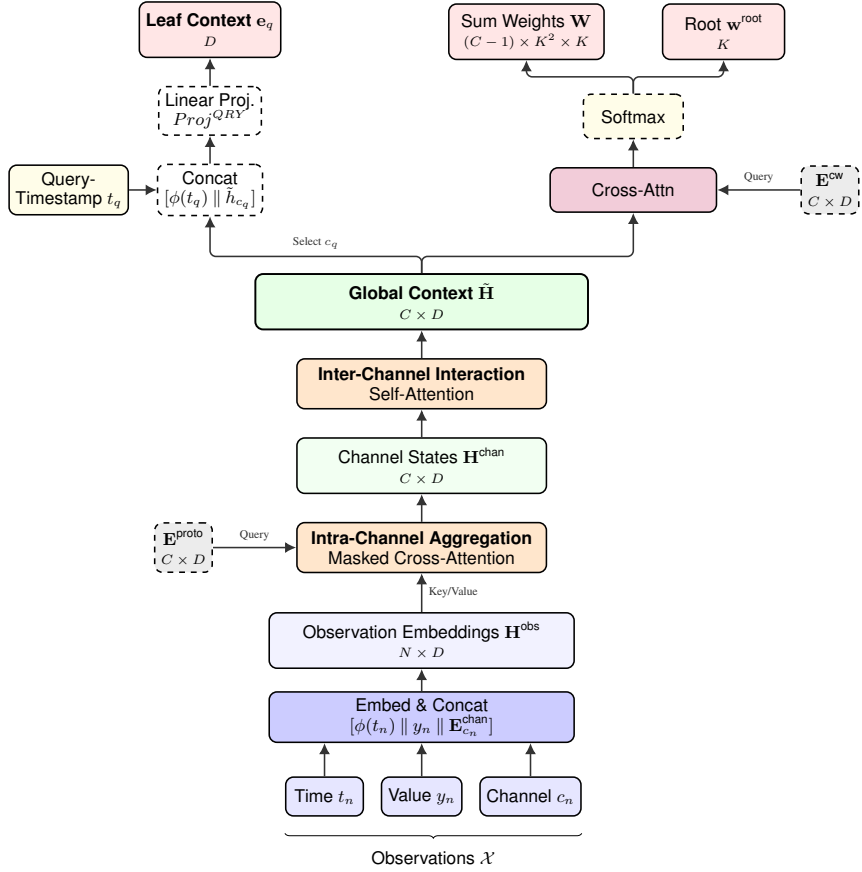
\begin{figure}[!htp]
\centering
\resizebox{\textwidth}{!}{%
\begin{tikzpicture}[
    node distance=0.8cm and 0.8cm,
    font=\sffamily\footnotesize,
    >={Latex[width=2mm,length=2mm]},
    base/.style = {draw, rounded corners, align=center, line width=0.8pt},
    input/.style = {base, fill=blue!10, rectangle, minimum height=0.7cm, minimum width=1.4cm},
    state/.style = {base, fill=green!10, rectangle, minimum width=4.5cm, minimum height=1.0cm},
    output/.style = {base, fill=red!10, rectangle, minimum width=2.5cm, minimum height=1.0cm},
    process/.style = {base, fill=blue!20, rectangle, minimum width=5.5cm, minimum height=0.9cm},
    attn_module/.style = {base, fill=orange!20, rectangle, minimum width=4.5cm, minimum height=0.9cm},
    hyper_module/.style = {base, fill=purple!20, rectangle, minimum width=3.0cm, minimum height=0.8cm},
    op_small/.style = {base, fill=white, rectangle, minimum height=0.8cm, minimum width=1.8cm, dashed},
    param/.style = {base, fill=gray!15, dashed, rectangle, minimum height=0.8cm},
    arrow/.style = {->, thick, rounded corners=3pt, color=black!80}
]

    
    \node[input] (val) {Value $y_n$};
    \node[input, left=0.3cm of val] (time) {Time $t_n$};
    \node[input, right=0.3cm of val] (chan) {Channel $c_n$};
    
    \draw[decorate, decoration={brace, mirror, raise=0.3cm}] (time.south west) -- (chan.south east) node[midway, below=0.5cm] {Observations $\mathcal{X}$};

    \node[process, above=0.6cm of val] (embed) {Embed \& Concat\\ $[\phi^\textsc{time}(t_n) \mathbin{\|} y_n \mathbin{\|} \mathbf{E}^{\text{chan}}_{c_n}]$};
    
    \draw[arrow] (time) -- (time|-embed.south);
    \draw[arrow] (val) -- (embed);
    \draw[arrow] (chan) -- (chan|-embed.south);

    \node[base, fill=blue!5, above=0.4cm of embed, minimum width=5.5cm, minimum height=1.0cm] (h_obs) {Observation Embeddings $\mathbf{H}^{\text{obs}}$\\ \scriptsize $L \times D$};
    \draw[arrow] (embed) -- (h_obs);

    \node[attn_module, above=0.7cm of h_obs] (intra_attn) {\textbf{Intra-Channel Aggregation}\\Masked Cross-Attention};
    \draw[arrow] (h_obs) -- node[right, font=\tiny] {Key/Value} (intra_attn);
    
    \node[param, left=1.5cm of intra_attn] (proto) {$\mathbf{E}^{\text{proto}}$\\ \scriptsize $C \times D$};
    \draw[arrow] (proto) -- node[above, font=\tiny] {Query} (intra_attn);

    \node[state, above=0.5cm of intra_attn, fill=green!5] (h_chan) {Channel States $\mathbf{H}^{\text{chan}}$\\ \scriptsize $C \times D$};
    \draw[arrow] (intra_attn) -- (h_chan);

    \node[attn_module, above=0.5cm of h_chan] (inter_attn) {\textbf{Inter-Channel Interaction}\\Self-Attention};
    \draw[arrow] (h_chan) -- (inter_attn);

    \node[state, above=0.5cm of inter_attn, line width=1pt, minimum width=6cm] (global_ctx) {\textbf{Global Context} $\tilde{\mathbf{H}}$\\ \scriptsize $C \times D$};
    \draw[arrow] (inter_attn) -- (global_ctx);


    \coordinate[above left=1.5cm and 0.8cm of global_ctx] (left_branch_anchor);
    
    \node[op_small, at=(left_branch_anchor)] (leaf_concat) {Concat\\$[\phi^\textsc{time}(t_n) \mathbin{\|} \tilde{h}_{c_n}]$};
    
    \node[input, left=0.5cm of leaf_concat, fill=yellow!10] (query_t) {Query- \\Timestamp $t_n$};
    \draw[arrow] (query_t) -- (leaf_concat);
    
    \draw[arrow] (global_ctx.north) -- ++(0,0.3) -| node[pos=0.25, above, font=\tiny] {Select $c_n$} (leaf_concat.south);

    \node[op_small, above=0.5cm of leaf_concat] (leaf_proj) {Linear Proj.\\$Proj^{QRY}$};
    \draw[arrow] (leaf_concat) -- (leaf_proj);

    \node[output, above=0.5cm of leaf_proj] (leaf_out) {\textbf{Leaf Context} $\mathbf{e}_q$\\ \scriptsize $D$};
    \draw[arrow] (leaf_proj) -- (leaf_out);

    \coordinate[above right=1.5cm and 0.8cm of global_ctx] (right_branch_anchor);

    \node[hyper_module, at=(right_branch_anchor)] (cross_attn) {Cross-Attn};
    
    \draw[arrow] (global_ctx.north) -- ++(0,0.3) -| (cross_attn.south);

    \node[param, right=1.5cm of cross_attn] (e_cw) {$\mathbf{E}^{\text{cw}}$\\ \scriptsize $C \times D$};
    \draw[arrow] (e_cw) -- node[above, font=\tiny] {Query} (cross_attn);

    \node[op_small, fill=yellow!10, above=0.5cm of cross_attn] (softmax) {Softmax};
    \draw[arrow] (cross_attn) -- (softmax);

    \node[output, above left=0.6cm and -0.4cm of softmax, minimum width=2.8cm] (w_sum) {Sum Weights $\mathbf{W}$\\ \scriptsize $(C-1) \times K^2 \times K$};
    \node[output, above right=0.6cm and -0.4cm of softmax, minimum width=2.2cm] (w_root) {Root $\mathbf{w}^{\text{root}}$\\ \scriptsize $K$};
    
    \draw[arrow] (softmax.north) -- ++(0,0.3) -| (w_sum.south);
    \draw[arrow] (softmax.north) -- ++(0,0.3) -| (w_root.south);

\end{tikzpicture}%
}
\caption{Detailed architecture of the \model{} Encoder with tensor dimensions. The global context $\tilde{\mathbf{H}}$ (dim $C \times D$) branches to generate leaf contexts (dim $D$) and circuit weights.}\label{fig:circuits_encoder_dims}
\end{figure}

%% file: appendix/circuit_sampling.tex
\subsection{Sampling with \model{}' Sum Product Network}\label{app:sampling_spn}
Let $k_{c} \in \{1, \dots, K\}$ denote the index of the active sum node (latent component) for channel $c$. The sampling proceeds as follows:

\textbf{Root Initialization ($c=C$):} We begin by sampling the final latent state $k_{C}$ directly from the distribution defined by the final weights:
\begin{equation}
    k_{C} \sim \text{Categorical}(\wroot)
\end{equation}
    
\textbf{Recursive Backtracking ($c = C, \dots, 2$):} Given the current active state $k_{c}=k$, we must determine the branch (product node) that generated it. The sum node $k$ aggregates over $K^2$ product terms formed by the previous channel state and the current channel leaf. We sample the flattened product index $m \in \{1, \dots, K^2\}$ using the mixing weights $W^{(c)}$ specific to the active node $k$:
\begin{equation}
    m \sim \text{Categorical}(W_{\cdot, k}^{(c)})
\end{equation}
The sampled index $m$ uniquely maps to a pair of indices: the latent state of the previous channel $k_{c-1}$ and the specific mixture component $j_c$ for the current channel's leaf distribution. We update the active state for the next step ($k_{c-1}$) and record the active leaf component index $j_c$ for the current channel.

\textbf{Termination ($c=1$):} At the final step of the backward pass, the remaining active state $k_{1}$ directly indicates the active leaf component $j_1$ for the first channel.

%% file: appendix/parallelizing_circuits.tex
\subsection{Parallelizing the Circuit Computation}\label{app:parallel_circuit}

The recursive Sum-Product Network (SPN) aggregation described in \Cref{sec:rec_spn} inherently suggests a sequential computation: the joint state $\phi^{(c)}$ depends on $\phi^{(c-1)}$, requiring $O(C)$ sequential steps for $C$ channels. However, by analyzing the algebraic structure of the update rule, we can reformulate the recurrence as a sequence of matrix multiplications in the log-semiring. This reformulation allows us to utilize parallel prefix scan algorithms (associative scans), reducing the time complexity from linear $O(C)$ to logarithmic $O(\log C)$ on parallel hardware (e.g., GPUs).

\paragraph{Connection to HMMs.}
The recursive update in \model{} is isomorphic to the Forward Algorithm in a Hidden Markov Model~\citep{Rabiner1989.Tutorial} (HMM) with a time-varying transition matrix. 
Recall the update rule for the circuit state. Let $\log \phi^{(c-1)} \in \mathbb{R}^K$ be the log-density vector of the joint state after $c-1$ channels. The update for the $k$-th component of the next channel $c$ is given by aggregating over the previous states $i$ and the current leaf mixture components $j$:
\begin{equation*}
    \log \phi_k^{(c)} = \log \sum_{i=1}^K \sum_{j=1}^K \exp \left(\log \phi_i^{(c-1)} + L_{c,j} + \log W_{(i,j), k}^{(c)} \right)
\end{equation*}
where $L_{c,j}$ is the log-likelihood of the $j$-th leaf component for channel $c$, and $W^{(c)}$ describes the mixing weights. We can rearrange the summation to isolate the dependence on the previous state $\phi_i^{(c-1)}$:
\begin{equation*}
   \log \phi_k^{(c)} = \log \sum_{i=1}^K \exp \left(\log \phi_i^{(c-1)} + \underbrace{\log \sum_{j=1}^K \exp \left( L_{c,j} + \log W_{(i,j), k}^{(c)} \right)}_{M_{ik}^{(c)}} \right) 
\end{equation*}
Here, we define a transition matrix $M^{(c)} \in \mathbb{R}^{K \times K}$ for step $c$. The entry $M_{ik}^{(c)}$ represents the log-probability of transitioning from joint component $i$ to component $k$, marginalizing over the leaf distributions of the current channel. 

In vector notation, this becomes a matrix-vector multiplication in the log-semiring (where $(\oplus, \otimes) \to (\text{logsumexp}, +)$):
$\phi^{(c)} = \phi^{(c-1)} \otimes_{\log} M^{(c)}$.

Consequently, the final joint distribution weights $\phi^{(C)}$ can be expressed as the product of the initial state $\phi^{(1)}$ and a sequence of transition matrices:
\begin{equation*}
    \phi^{(C)} = \phi^{(1)} \otimes_{\log} M^{(2)} \otimes_{\log} M^{(3)} \dots \otimes_{\log} M^{(C)}
\end{equation*}
This formulation highlights that \model{}, while structurally an SPN, shares the computational characteristics of linear recurrent models and HMMs~\citep{Sarkka2021.Temporal}.

\paragraph{Parallel Log-Semiring Scan.}
Since matrix multiplication is associative, the cumulative product of the transition matrices $M^{(c)}$ can be computed in parallel using the Hillis-Steele algorithm~\citep{Hillis1986.Data}.
Instead of computing the product sequentially ($((M^{(2)} M^{(3)}) M^{(4)}) \dots$), a parallel scan computes the products in a tree-like structure. For a sequence of length $C$, this requires $\mathcal{O}(\log C)$ sequential steps with sufficient parallel threads.

%% file: appendix/bifurcation_details.tex
\section{Details on the Synthetic Bifurcation Experiment}~\label{app:bifurcation}
For the toy experiment described in \Cref{sec:exp}, we generated IMTS trajectories using 50 evenly spaced time steps. The data points prior to bifurcation are randomly distributed around 0 with a standard deviation of 0.1. After the bifurcation point (at the 18th time step), the channels follow a Brownian motion with a drift of either 0.1 or -0.1 per time step. The standard deviation is still 0.1. This decision is random, independent per channel, and equiprobable for both paths.

The models are tasked with predicting the distribution of the final 75\% of time steps, conditioned on the initial 25\%. Consequently, the observations provided as input contain no information regarding the future drift direction. To simulate irregularity, we randomly dropped 5\% of the observations and query points.

We present samples from ProFITi in \Cref{fig:app_bifurcation}, along with the final normalized joint negative log-likelihood (njNLL) on the test split for ProFITi, MOSES, and \model{}. While ProFITi correctly captures the Brownian motion dynamics and approaches \model{} in modeling the joint likelihood (\Cref{tab:nll_bifurcation}), it fails to learn the strict separation of modalities.
\begin{table}[h!]
\caption{Test njNLL for the synthetic four channel bifurcation task.}\label{tab:nll_bifurcation}
\small
\centering
\begin{tabular}{lcccc}
    \toprule
    & MOSES-10 & MOSES-16 & ProFITi & \model{} \\
    \midrule
    njNLL & -0.340 ±.07 & -0.290 ± .04 & -1.253 ± .01 & \textbf{-1.280 ± .02} \\
    \bottomrule
\end{tabular}
\end{table}

\input{figures/bifurcation_app.tex}

%% file: figures/bifurcation_app.tex
\begin{figure}[h!]
    \centering
    \includegraphics[width=1.5in]{figures/bifurcation/ground_truth.pdf}
    \includegraphics[width=1.5in]{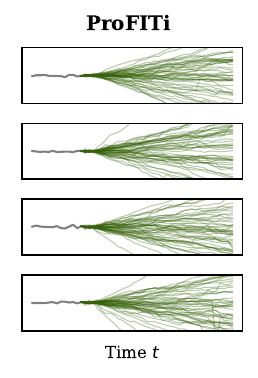}
    \includegraphics[width=1.5in]{figures/bifurcation/circ.pdf}
    \caption{Comparing samples from ProFITi and \model{} for the four channel bifurcation task.}\label{fig:app_bifurcation}
\end{figure}

%% file: appendix/datasets.tex
\section{Dataset Descriptions}~\label{app:datasets}
We utilize the following five benchmark datasets to evaluate performance, covering climatological, physiological, and human motion domains. The first four are used in coherence with \citet{Yalavarthi2025.Reliable}. The specific characteristics and preprocessing protocols for each are detailed below:

    \textbf{USHCN.} Derived from the U.S. Historical Climatology Network~\citep{Menne2006.US}, this dataset tracks climate changes across the USA\@. It consists of daily measurements of 5 variables (snow precipitation, rain precipitation, snow depth, min. temperature, and max temperature) collected over 150 years from 1,218 meteorological stations. Following the protocol of \citet{DeBrouwer2019.GRUODEBayes}, we selected a subset of 1,100 stations and an observation window of 4 years (1996-2000). USHCN is transformed into an IMTS by randomly dropping 95\% recorded observations. 

    \textbf{PhysioNet2012.} Sourced from the PhysioNet Computing in Cardiology Challenge 2012~\citep{Silva2012.Predicting}, this physiological dataset contains medical records of 12,000 patients admitted to the ICU. It tracks 37 vital signs over a 48-hour period. Consistent with \citet{Che2018.Recurrent}, the dataset is binned into hourly observations by rounding time stamps to the hour and averaging observations, if a channel is observed multiple times within one hour.

    \textbf{MIMIC-III.} The Medical Information Mart for Intensive Care III~\citep{Johnson2016.MIMICIII} is a large-scale physiological dataset collected at Beth Israel Deaconess Medical Center. It contains 17,874 clinical instances measuring 96 variables over 48 hours. Following \citet{DeBrouwer2019.GRUODEBayes}, observations are aggregated and rounded to 30-minute intervals.

     \textbf{MIMIC-IV.} This dataset features records from a tertiary academic medical center in Boston. It tracks 102 variables from ICU patients over a 48-hour period~\citep{PhysioNet-mimiciv-2.2,Johnson2023.MIMICIV}. Adopting the standards of \citet{Bilos2021.Neural}, the observations are rounded to 1-minute intervals. Due to the more fine-grained binning, this dataset records the greatest number of observations and queries.

     \textbf{Human Activity.} This dataset~\citep{Kaluza2010.Agentbased,Vidulin2010.Localization} contains 3D position recordings of four body-worn tags (left ankle, right ankle, chest, and belt) from five people, each performing a scripted sequence of activities such as walking, sitting, and lying five times, which yields 25 recordings with 12 channels. Only one tag is read out at a time, so the resulting multivariate series is sporadically observed. Following~\citet{Rubanova2019.Latent} and \citet{Zhang2024.Irregular}, timestamps are quantized to milliseconds and repeated readings of the same tag within one millisecond are averaged. Each recording is cut into non-overlapping windows of 4 seconds, which take the role of the 48 hours of the medical datasets: in the 36--12 setting, for instance, the models observe the first 3000\,ms of a window and forecast the remaining 1000\,ms. The splits are made on the level of recordings, so that no window of a test recording leaks into training.

\begin{table}[h]
    \centering
    \caption{Summary of Dataset Characteristics and Preprocessing}\label{tab:datasets}
    \small
    \vspace{0.2cm}
    \begin{tabular}{l l c c c c c }
        \toprule
        \textbf{Dataset} &  \textbf{Samples} & \textbf{Channels} & \textbf{Duration} & \textbf{Binning} & \textbf{Max. Observations}\\
        \midrule
        \textbf{USHCN}           &  1,100 (Stations) & 5   & 4 Years  & 1 Week                                   &  325 \\
        \textbf{PhysioNet2012}   &  12,000 (Patients) & 37  & 48 Hours & 1 Hour                                  &  552 \\
        \textbf{MIMIC-III}       &  21,000 (Patients) & 96  & 48 Hours & 30 Mins                                 &  741 \\
        \textbf{MIMIC-IV}       &  18,000 (Patients) & 102 & 48 Hours & 1 Min                                   &  1386 \\
        \textbf{Human Activity} &  1,360 (Windows) & 12 & 4 Seconds & 1 ms                                  &  387 \\
        \bottomrule
    \end{tabular}
\end{table}

\begin{table}[h]
    \centering
    \caption{Minimum, Average, and Maximum Number of Queries (Grouped by Task)}\label{tab:datasets_grouped}
    
    \setlength{\tabcolsep}{0.6mm}

    \begin{tabular}{l ccc !{\hspace{0.35cm}} ccc !{\hspace{0.35cm}} ccc !{\hspace{0.35cm}} ccc}
        \toprule
        & \multicolumn{3}{c}{\textbf{36-3}} 
        & \multicolumn{3}{c}{\textbf{36-12}} 
        & \multicolumn{3}{c}{\textbf{24-24}} 
        & \multicolumn{3}{c}{\textbf{12-36}} \\
        
        \cmidrule(lr){2-4} \cmidrule(lr){5-7} \cmidrule(lr){8-10} \cmidrule(lr){11-13}
        
        \textbf{Dataset} & {min.} & avg. & max. 
                         & {min.} & avg. & max. 
                         & {min.} & avg. & max. 
                         & {min.} & avg. & max. \\
        \midrule
        \textbf{USHCN}         
        & 3 & 3.3 & 6    
        & 5 & 54.7 & 63   
        & 19 & 109.4 & 120  
        & 19 & 164.1 & 179 \\
        \textbf{PhysioNet2012} 
        & 1 & 19.8 & 53  
        & 1 & 61.8 & 148  
        & 1 & 139.9 & 297  
        & 1 & 220.7 & 459 \\
        \textbf{MIMIC-III}     
        & 1 & 10.3 & 85  
        & 1 & 25.3 & 264  
        & 1 & 67.2 & 512  
        & 1 & 109.7 & 785 \\
        \textbf{MIMIC-IV}      
        & 1 & 7.8 & 79   
        & 1 & 57.2 & 481  
        & 1 & 128.9 & 906  
        & 1 & 208.4 & 1357 \\
        \textbf{Human Activity}
        & 3 & 9.0 & 9
        & 3 & 90.3 & 99
        & 33 & 180.2 & 195
        & 21 & 269.9 & 291 \\
        \bottomrule
    \end{tabular}
\end{table}

%% file: appendix/hyperparameters.tex
\section{Hyperparameters}~\label{app:hyper}
Each model is trained for a maximum of 2000 epochs using early stopping with patience of 30 epochs. We use the AdamW~\citep{Kingma2017.Adam,loshchilov2019decoupledweightdecayregularization} optimizer with an initial learning rate of 0.001, an L2 regularization of 0.001 and a batch size of 64. Additionally, we use a learning-decay-on-plateau scheduler with a decay rate of 0.5 and 5 epochs patience. For each model we evaluate a maximum of 10 randomly searched hyperparameter setting. We train each setting on single fold and select the setting with the best validation njNLL.\@ We follow \citeauthor{Yalavarthi2025.Reliable} and select the following hyperparameter search spaces for the evaluated models:
\paragraph{MOSES.}For the MOSES model~\citep{Yalavarthi2025.Reliable}, we perform a search over the number of \textbf{mixture components} $\in \{1, 2, 5, 7, 10\}$. Additionally, we evaluate different sizes for the \textbf{hidden dimension} $\in \{16, 32, 64, 128\}$ and vary the number of \textbf{attention heads} $\in \{1, 2, 4\}$.
\paragraph{ProFITi.} For ProFITi~\citep{Yalavarthi2025.Probabilistic}, we evaluate the number of \textbf{flow layers} $\in \{7, 8, 9, 10\}$ and the \textbf{latent dimension size} $\in \{32, 64, 128, 256\}$.
\paragraph{Neural Flows.} For the Neural Flows model~\citep{Bilos2021.Neural}, we primarily search over the number of flow layers $\in \{1, 4\}$. We fix the \textbf{hidden dimension} to $64$, and the number of \textbf{hidden layers} is $2$, and the \textbf{flow model} is set to GRU.

\paragraph{\model{}.} For our model, we vary the number of \textbf{(components $K$)} $\in \{2, 3, 4\}$ and the \textbf{hidden dimension $D^\prime$} $\in \{32, 64\}$. For the Deep sigmoidal flow on the univariate marginals we tune the number of \textbf{flow layers} $ \in \{2,3\}$, and the \textbf{flow hidden dimension} $\in \{10,20\}$. The MLP that infers the parameters has 2 hidden layers and a hidden dimension of 32. The multi-head attention layers in the encoder are applied by setting the number of \textbf{attention heads} to $2$. For \model{} we had to set the batch size to 32 for the MIMIC-IV due to memory overflow.

The remaining hyperparameters of the baseline models are taken from the respective papers.

%% file: appendix/abl_details.tex
\section{Details on Ablation Studies}~\label{app:details_abl}
We provide additional information for the ablation studies presented in \Cref{tab:abl_spn}. 
\paragraph{w/o SPN.} We conduct this experiment by training a model where the number of leaf components $K$ is set to 1. 
Consequently, there is no weighted sum of different components; instead, channels are combined via a single product distribution.

\paragraph{w/o DSF.} We replace the Deep Sigmoidal Flow (DSF), which is applied to the univariate marginals, with 
normal distributions parameterized by ($\mu, \sigma$) that are inferred from the embeddings of the forecasting queries.
Formally:
\begin{align}
    \mu^k_i    &= f^{\text{mean}}(e^k_i) \\
    \sigma^k_i &= \exp(f^{\text{var}}(e^k_i)), 
\end{align}
where $f^{\text{mean}}$ and $f^{\text{var}}$ are implemented as 2-layer MLPs. We do not change the computation of the correlation matrix $R$. 

\paragraph{w/o GC.} For this ablation, we set the correlation matrix $R=I$, thereby treating the univariate marginals of a single component $k$ as independent.

\paragraph{w/ TF.} We utilize Tripletformer (TF)~\citep{Yalavarthi2023.Tripletformer} 
as an encoder to compute the circuit weights and the embeddings of the forecasting queries. 
In contrast to the encoder in \model{}, Tripletformer does not group observations/queries by channel. 
Instead, it first performs self-attention among all observation triplets (channel, timestamp, value)
and subsequently applies cross-attention from the encoded observation triplets to the forecasting queries. 

To infer the weights of the SPN, we include an additional Cross-Attention layer with learnable circuit weight embeddings $\mathbf{E}^\text{cw}$, (similar to our encoder but using all encoded observation triplets as attention keys and values).

%% file: appendix/additional_metrics.tex
\section{Additional Metrics}\label{app:additional_metrics}
\subsection{Energy Score and CRPS}\label{app:crps_es}
To evaluate the quality of samples produced by the learned distributions, we employ the Energy Score and the Continuous Ranked Probability Score (CRPS). Since closed-form solutions are often intractable for complex joint distributions modeled by \model{} and the baselines, we approximate these metrics using Monte Carlo estimation with $S=1000$ samples.

\paragraph{Energy Score.} The Energy Score is a multivariate strictly proper scoring rule that assesses the quality of the joint distribution, capturing correlations between variables in the IMTS.\@ For a ground truth vector $y \in \mathbb{R}^M$ corresponding to the query $\mathcal{Q}$, and $S$ samples, the empirical Energy Score is defined as:
\begin{equation}
    ES(y, \mathcal{Q}) = \frac{1}{S} \sum_{s=1}^{S} \| \hat{y}^{(s)} - y \|_2 - \frac{1}{2S^2} \sum_{s=1}^{S} \sum_{j=1}^{S} \| \hat{y}^{(s)} - \hat{y}^{(j)} \|_2
\end{equation}
where $\| \cdot \|_2$ denotes the Euclidean norm. The first term measures the distance between the samples and the ground truth (calibration), while the second term measures the diversity of the samples (sharpness).

\paragraph{Continuous Ranked Probability Score (CRPS).} While the Energy Score evaluates the joint distribution, the CRPS evaluates the accuracy of the univariate marginal distributions. We report the mean CRPS averaged over all query targets in $\mathcal{Q}$. The sample-based approximation for the CRPS is given by:
\begin{equation}
    CRPS(y, \mathcal{Q}) = \frac{1}{N} \sum_{i=1}^{N} \left( \frac{1}{S} \sum_{s=1}^{S} | \hat{y}_i^{(s)} - y_i | - \frac{1}{2S^2} \sum_{s=1}^{S} \sum_{j=1}^{S} | \hat{y}_i^{(s)} - \hat{y}_i^{(j)} | \right)
\end{equation}
where $y_i$ is the ground truth value for the $i$-th query target and $\hat{y}_i^{(s)}$ is the corresponding value in the $s$-th sample. \Cref{tab:crps} and \Cref{tab:energy} compare these scores for NeuralFlows, ProFITi, MOSES, and \model{}.

\input{tables/energy_crps.tex}

\subsection{Marginal NLL on the Main Benchmark}\label{app:mnll}
The marginal Negative Log Likelihood (mNLL) averages the univariate log-likelihoods over the $N$ query targets:
\begin{equation}
    \text{mNLL}(y|\mathcal{Q}, \mathcal{X}) = \frac{\sum_{n=1}^{N} -\log p(y_n \mid (t^\text{qry}_n,c^\text{qry}_n),\mathcal{X})}{N}.
\end{equation}
While the main paper reports the joint njNLL (\Cref{tab:njnll}), we additionally evaluate the mNLL on the same datasets and forecasting splits. Results are shown in \Cref{tab:mnll}.
\input{tables/mnll_table.tex}

\subsection{MSE of the Predicted Means}\label{app:mse}
To assess the point accuracy of the predicted distributions, \Cref{tab:mse} reports the Mean Squared Error (MSE) of the empirical means, calculated with 1000 samples, on the 36--3$^\star$ task. It includes the additional baselines of \Cref{app:add_baselines}.
\input{tables/mse.tex}

%% file: tables/energy_crps.tex
\begin{table}[!ht]
\centering
\caption{Energy Score of competing models. Lower values are better. Best results are in \textbf{bold}, second best in \textit{italics}. Models are trained and tuned on njNLL.}\label{tab:energy}
\setlength{\tabcolsep}{2mm}
\small
\begin{tabular}{clccccc}
    \toprule
    & Model & USHCN & PhysioNet'12 & MIMIC-III & MIMIC-IV & Activity \\
    \midrule
    \multirow{4}{*}{\rotatebox{90}{36--3$^\star$}}
    & NeuralFlows & 0.661 ± 0.06 & 1.691 ± 0.00 & \textit{1.381 ± 0.03} & 0.982 ± 0.01 & 3.801 ± 0.21 \\
    & ProFITi     & \textbf{0.452 ± 0.04} & \textbf{0.879 ± 0.30} & 1.606 ± 0.17 & \textbf{0.808 ± 0.00} & \textbf{0.793 ± 0.06} \\
    & MOSES       & 0.552 ± 0.04 & \textit{1.599 ± 0.01} & \textbf{1.353 ± 0.03} & \textit{0.906 ± 0.03} & 0.918 ± 0.11 \\
    & \model{}    & \textit{0.468 ± 0.02} & 1.613 ± 0.00 & 1.489 ± 0.04 & 0.927 ± 0.05 & \textit{0.796 ± 0.07} \\
    \midrule
    \multirow{4}{*}{\rotatebox{90}{36--12}}
    & NeuralFlows & 5.943 ± 0.39 & 6.107 ± 0.18 & 8.107 ± 0.83 & 12.222 ± 0.39 & 12.225 ± 0.45 \\
    & ProFITi     & \textit{4.385 ± 1.7} & 3.575 ± 0.45 & 2.546 ± 0.13 & 5.394 ± 0.15 & \textit{3.328 ± 0.15} \\
    & MOSES       & \textbf{3.252 ± 0.13} & \textbf{3.195 ± 0.04} & \textit{2.358 ± 0.10} & \textit{2.869 ± 0.04} & 3.332 ± 0.09 \\
    & \model{}    & 5.535 ± 0.93 & \textit{3.213 ± 0.00} & \textbf{2.279 ± 0.01} & \textbf{2.835 ± 0.09} & \textbf{3.174 ± 0.09} \\
    \midrule
    \multirow{4}{*}{\rotatebox{90}{24--24}}
    & NeuralFlows & \textit{8.448 ± 0.66} & 9.522 ± 0.44 & 10.290 ± 0.35 & 16.836 ± 0.71 & 15.425 ± 1.19 \\
    & ProFITi     & 14.42 ± 7.6 & 5.742 ± 0.28 & 4.287 ± 0.38 & 10.82 ± 1.1 & 6.171 ± 0.25 \\
    & MOSES       & \textbf{5.330 ± 0.31} & \textit{5.284 ± 0.03} & \textit{3.632 ± 0.05} & \textbf{4.268 ± 0.07} & \textit{5.207 ± 0.25} \\
    & \model{}    & 10.71 ± 1.9 & \textbf{5.267 ± 0.01} & \textbf{3.543 ± 0.04} & \textit{4.482 ± 0.15} & \textbf{5.191 ± 0.16} \\
    \midrule
    \multirow{4}{*}{\rotatebox{90}{12--36}}
    & NeuralFlows & \textbf{10.79 ± 1.1} & 12.170 ± 0.41 & 14.524 ± 0.95 & 23.049 ± 0.97 & 18.529 ± 0.83 \\
    & ProFITi     & 41.61 ± 39 & 8.439 ± 0.42 & 6.567 ± 0.53 & 11.10 ± 0.91 & 9.274 ± 0.91 \\
    & MOSES       & 21.22 ± 17 & \textit{7.685 ± 0.63} & \textit{5.770 ± 0.03} & \textit{6.276 ± 0.03} & \textit{6.980 ± 0.17} \\
    & \model{}    & \textit{12.18 ± 1.8} & \textbf{7.241 ± 0.01} & \textbf{5.730 ± 0.06} & \textbf{5.851 ± 0.05} & \textbf{6.808 ± 0.23} \\
    \bottomrule
\end{tabular}
\end{table}

\begin{table}[!ht]
\centering
\caption{CRPS of competing models. Lower values are better. Best results are in \textbf{bold}, second best in \textit{italics}. Models are trained and tuned on njNLL.}\label{tab:crps}
\setlength{\tabcolsep}{2mm}
\small
\begin{tabular}{clccccc}
    \toprule
    & Model & USHCN & PhysioNet'12 & MIMIC-III & MIMIC-IV & Activity \\
    \midrule
    \multirow{4}{*}{\rotatebox{90}{36--3$^\star$}}
    & NeuralFlows & 0.306 ± 0.03 & 0.277 ± 0.00 & 0.308 ± 0.00 & 0.281 ± 0.00 & 0.194 ± 0.01 \\
    & ProFITi     & \textbf{0.182 ± 0.01} & 0.271 ± 0.00 & 0.319 ± 0.00 & 0.279 ± 0.01 & \textbf{0.185 ± 0.01} \\
    & MOSES       & 0.220 ± 0.02 & \textit{0.260 ± 0.00} & \textit{0.296 ± 0.01} & \textit{0.245 ± 0.01} & 0.211 ± 0.02 \\
    & \model{}    & \textbf{0.182 ± 0.01} & \textbf{0.252 ± 0.00} & \textbf{0.286 ± 0.01} & \textbf{0.221 ± 0.00} & \textit{0.186 ± 0.01} \\
    \midrule
    \multirow{4}{*}{\rotatebox{90}{36--12}}
    & NeuralFlows & 0.428 ± 0.02 & 0.320 ± 0.00 & 0.384 ± 0.02 & 0.318 ± 0.01 & 0.239 ± 0.01 \\
    & ProFITi     & \textit{0.234 ± 0.04} & 0.312 ± 0.03 & 0.383 ± 0.03 & 0.654 ± 0.08 & \textit{0.228 ± 0.01} \\
    & MOSES       & \textbf{0.211 ± 0.01} & \textbf{0.285 ± 0.00} & \textit{0.350 ± 0.01} & \textit{0.262 ± 0.01} & 0.235 ± 0.01 \\
    & \model{}    & 0.359 ± 0.06 & \textit{0.286 ± 0.00} & \textbf{0.327 ± 0.01} & \textbf{0.249 ± 0.01} & \textbf{0.226 ± 0.01} \\
    \midrule
    \multirow{4}{*}{\rotatebox{90}{24--24}}
    & NeuralFlows & 0.433 ± 0.02 & 0.352 ± 0.00 & 0.407 ± 0.01 & 0.333 ± 0.01 & 0.310 ± 0.02 \\
    & ProFITi     & \textit{0.355 ± 0.06} & 0.336 ± 0.01 & 0.438 ± 0.07 & 1.132 ± 0.17 & 0.313 ± 0.01 \\
    & MOSES       & \textbf{0.239 ± 0.01} & \textit{0.318 ± 0.00} & \textit{0.333 ± 0.00} & \textbf{0.279 ± 0.01} & \textbf{0.266 ± 0.01} \\
    & \model{}    & 0.513 ± 0.11 & \textbf{0.315 ± 0.00} & \textbf{0.323 ± 0.01} & \textit{0.281 ± 0.01} & \textit{0.268 ± 0.01} \\
    \midrule
    \multirow{4}{*}{\rotatebox{90}{12--36}}
    & NeuralFlows & \textbf{0.437 ± 0.02} & 0.396 ± 0.01 & 0.477 ± 0.00 & 0.356 ± 0.00 & 0.344 ± 0.02 \\
    & ProFITi     & 0.593 ± 0.36 & 0.408 ± 0.03 & 2.413 ± 1.0 & 0.823 ± 0.12 & 0.384 ± 0.04 \\
    & MOSES       & 1.198 ± 1.15 & \textit{0.379 ± 0.04} & \textit{0.418 ± 0.00} & \textit{0.332 ± 0.00} & \textit{0.298 ± 0.01} \\
    & \model{}    & \textit{0.456 ± 0.07} & \textbf{0.355 ± 0.00} & \textbf{0.409 ± 0.00} & \textbf{0.292 ± 0.00} & \textbf{0.292 ± 0.01} \\
    \bottomrule
\end{tabular}
\end{table}

%% file: tables/mnll_table.tex
\begin{table}[!ht]
\centering
\caption{Comparison of mNLL for different forecasting splits. Lower values are better. Best results are in \textbf{bold}, second best in \textit{italics}. The $\star$ represents that the forecasting horizon includes the next $n$ steps, instead of $n$-many hours/months.}\label{tab:mnll}

\setlength{\tabcolsep}{2mm}

\small

\begin{tabular}{clrrrrr}
    \toprule
    & Model & \multicolumn{1}{c}{USHCN} & \multicolumn{1}{c}{PhysioNet'12} & \multicolumn{1}{c}{MIMIC-III} & \multicolumn{1}{c}{MIMIC-IV} & \multicolumn{1}{c}{Activity} \\
    \midrule
    \multirow{4}{*}{\rotatebox{90}{36--3$^\star$}}
    & NeuralFlows & 0.775 ± 0.18 & 0.492 ± 0.03 & 0.866 ± 0.10 & 0.796 ± 0.05 & 0.360 ± 0.16 \\
    & ProFITi & \textit{-3.324 ± 0.21} & -0.016 ± 0.09 & 0.408 ± 0.03 & \textit{0.500 ± 0.32} & \textit{0.118 ± 0.08} \\
    & MOSES & -3.355 ± 0.16 & \textit{-0.271 ± 0.03} & \textit{0.163 ± 0.03} & -0.634 ± 0.02 & 0.139 ± 0.05 \\
    & \model{} & \textbf{-3.717 ± 0.20} & \textbf{-0.287 ± 0.04} & \textbf{0.095 ± 0.11} & \textbf{-0.731 ± 0.06} & \textbf{0.060 ± 0.03} \\
    \midrule
    \multirow{4}{*}{\rotatebox{90}{36--12}}
    & NeuralFlows & 1.344 ± 0.13 & 0.709 ± 0.05 & 1.380 ± 0.13 & 1.073 ± 0.04 & 0.713 ± 0.08 \\
    & ProFITi     & \textbf{-3.092 ± 0.40} & 1.376 ± 1.8 & 1.158 ± 0.26 & 0.550 ± 0.15 & 0.477 ± 0.09 \\
    & MOSES       & -1.428 ± 0.98 & \textit{-0.083 ± 0.03} & \textit{0.624 ± 0.08} & \textit{-0.111 ± 0.06} & \textit{0.313 ± 0.03} \\
    & \model{}    & \textit{-3.072 ± 0.25} & \textbf{-0.118 ± 0.03} & \textbf{0.140 ± 0.03} & \textbf{-0.602 ± 0.04} & \textbf{0.311 ± 0.03} \\
    \midrule
    \multirow{4}{*}{\rotatebox{90}{24--24}}
    & NeuralFlows & 1.436 ± 0.20 & 1.065 ± 0.08 & 1.324 ± 0.06 & 1.171 ± 0.03 & 1.061 ± 0.12 \\
    & ProFITi     & -0.626 ± 0.16 & 0.705 ± 0.18 & 1.101 ± 0.28 & 0.609 ± 0.03 & 1.315 ± 0.36 \\
    & MOSES       & \textit{-1.917 ± 0.11} & \textit{-0.020 ± 0.06} & \textit{0.700 ± 0.02} & \textit{-0.028 ± 0.02} & \textit{0.506 ± 0.04} \\
    & \model{}    & \textbf{-2.593 ± 0.39} & \textbf{-0.050 ± 0.01} & \textbf{0.150 ± 0.07} & \textbf{-0.345 ± 0.10} & \textbf{0.500 ± 0.02} \\
    \midrule
    \multirow{4}{*}{\rotatebox{90}{12--36}}
    & NeuralFlows & 1.458 ± 0.27 & 1.439 ± 0.20 & 1.449 ± 0.03 & 1.195 ± 0.04 & 1.169 ± 0.10 \\
    & ProFITi     & -0.116 ± 0.78 & 2.977 ± 3.0 & \textit{0.341 ± 0.25} & 0.912 ± 0.14 & 1.878 ± 0.53 \\
    & MOSES       & \textit{-1.431 ± 0.71} & \textbf{0.040 ± 0.13} & 0.850 ± 0.02 & \textit{0.218 ± 0.00} & \textit{0.681 ± 0.02} \\
    & \model{}    & \textbf{-2.254 ± 0.34} & \textit{0.162 ± 0.05} & \textbf{0.016 ± 0.05} & \textbf{-0.302 ± 0.02} & \textbf{0.518 ± 0.04} \\
\bottomrule
\end{tabular}
\end{table}

%% file: tables/mse.tex
\begin{table}[!ht]
\centering
\caption{Comparing the MSE of the mean estimated by 1000 samples from probabilistic models on the 36--3$^{\star}$ setting.}\label{tab:mse}
\small
\setlength{\tabcolsep}{5mm} 
\begin{tabular}{lcccc}
    \toprule
    Model & USHCN & PhysioNet’12 & MIMIC-III & MIMIC-IV \\
    \midrule
    ProFITi          & 0.308 ± .06 & 0.305 ± .01 & 0.548 ± .06 & 0.389 ± .02 \\
    GRU-ODE          & 0.410 ± .11 & 0.329 ± .00 & \textit{0.479 ± .04} & 0.365 ± .01 \\
    NeuralFlows     & 0.424 ± .11 & 0.331 ± .01 & \textit{0.479 ± .05} & 0.374 ± .02 \\
    CRU              & \textbf{0.290 ± .06} & 0.475 ± .02 & 0.725 ± .04 & OOM \\
    GPR              & 0.597 ± .11 & 0.575 ± .06 & 0.862 ± .02 & 0.609 ± .01 \\
    GMM              & \textit{0.294 ± .08} & \textit{0.293 ± .01} & 0.535 ± .06 & \textit{0.332 ± .02} \\
    MOSES            & 0.411 ± .10 & 0.307 ± .01 & 0.517 ± .06 & 0.342 ± .03 \\
    \model{} & 0.306 ± .03 & \textbf{0.292 ± .00} & \textbf{0.475 ± .05} & \textbf{0.285 ± .00} \\
    \bottomrule
\end{tabular}
\end{table}

%% file: appendix/additional_baselines.tex
\section{Additional Baselines}\label{app:add_baselines}
In addition to the baselines presented in \Cref{sec:exp}, we compare \model{} to some additional baselines.
Similar to NeuralFlows~\citep{Bilos2021.Neural}, Continuous Recurrent Units (CRU)~\citep{Schirmer2022.Modeling} and GRU-ODE-Bayes~\citep{DeBrouwer2019.GRUODEBayes}
learn to forecasting univariate marginals of an IMTS\@.

Furthermore, we compare with Gaussian Process Regression~\citep{Durichen2015.Multitask} and Gaussian Mixture Model~\citep{Yalavarthi2025.Reliable}, which predicts a mixture of multivariate Gaussian in the same way that MOSES~\citep{Yalavarthi2025.Reliable}, however without applying normalizing flows.
The njNLL and mNLL on the 36--3$^\star$ task are shown in \Cref{tab:additional_nll}; the MSE of the predicted means is reported in \Cref{tab:mse}.
\input{tables/add_base_lines.tex}

%% file: tables/add_base_lines.tex
\begin{table}[!ht]
\centering
\caption{Comparison of njNLL and mNLL on the 36-3$^\star$ task. Lower values are better. Best results are in \textbf{bold}, second best in \textit{italics}. Results are taken from~\citep{Yalavarthi2025.Reliable}. OOM stands for out-of-memory.}\label{tab:additional_nll}
\setlength{\tabcolsep}{3mm}
\small
\begin{tabular}{clrrrr}
    \toprule
    & Model & \multicolumn{1}{c}{USHCN} & \multicolumn{1}{c}{PhysioNet'12} & \multicolumn{1}{c}{MIMIC-III} & \multicolumn{1}{c}{MIMIC-IV} \\
    \midrule
    \multirow{8}{*}{\rotatebox{90}{njNLL}}
    & ProFITi     & -3.226 ± .23 & \textbf{-0.647 ± .08} & \textit{-0.377 ± .03} & \textit{-1.777 ± .07} \\
    & GRU-ODE     & 0.766 ± .16 & 0.501 ± .00 & 0.961 ± .06 & 0.823 ± .32 \\
    & NeuralFlows & 0.775 ± .15 & 0.496 ± .00 & 0.998 ± .18 & 0.689 ± .09 \\
    & CRU         & 0.761 ± .19 & 1.057 ± .01 & 1.234 ± .08 & OOM \\
    & GPR         & 2.011 ± 1.4 & 1.367 ± .07 & 3.146 ± .36 & 2.789 ± .05 \\
    & GMM         & 1.050 ± .03 & 1.063 ± .00 & 1.160 ± .02 & 1.076 ± .00 \\
    & MOSES       & \textit{-3.357 ± .18} & -0.491 ± .04 & -0.305 ± .03 & -1.668 ± .10 \\
    & \model{}    & \textbf{-3.789 ± .22} & \textit{-0.550 ± .01} & \textbf{-0.574 ± .08} & \textbf{-2.113 ± .04} \\
    \midrule
    \multirow{8}{*}{\rotatebox{90}{mNLL}}
    & ProFITi     & \textit{-3.324 ± .21} & \textit{-0.016 ± .09} & 0.408 ± .03 & \textit{0.500 ± .32} \\
    & GRU-ODE     & 0.776 ± .17 & 0.504 ± .06 & 0.839 ± .03 & 0.876 ± .59 \\
    & NeuralFlows & 0.775 ± .18 & 0.492 ± .03 & 0.866 ± .10 & 0.796 ± .05 \\
    & CRU         & 0.762 ± .18 & 0.931 ± .02 & 1.209 ± .04 & OOM \\
    & GPR         & 1.235 ± .10 & 1.161 ± .07 & 1.341 ± .01 & 1.161 ± .01 \\
    & GMM         & 1.042 ± .02 & 1.069 ± .00 & 1.124 ± .01 & 1.075 ± .01 \\
    & MOSES       & -3.355 ± .16 & -0.271 ± .03 & \textit{0.163 ± .03} & -0.634 ± .02 \\
    & \model{}    & \textbf{-3.717 ± .20} & \textbf{-0.287 ± .04} & \textbf{0.095 ± .11} & \textbf{-0.731 ± .06} \\
    \bottomrule
\end{tabular}
\end{table}

%% file: appendix/channel_ordering.tex
\section{Importance of Channel Ordering}\label{app:order}
To quantify the importance of channel ordering, we repeated the PhysioNet'12 evaluation on the 36--12 task using five randomly sampled channel permutations. The resulting test njNLL values are provided in \Cref{tab:perm}. While the results indicate that ordering has a statistically discernible impact, the variation is negligible relative to the performance gaps between the competing models. Consequently, the choice of permutation does not alter the qualitative conclusions of the experiment.
\begin{table}[!ht]
    \centering
    \caption{Permuting the channels order in the recursive SPN aggregation for the Physionet'12 36--12 task.}\label{tab:perm}
    \begin{tabular}{lcccccc}
        \toprule
        & Original & Perm. 1 & Perm. 2 & Perm. 3 & Perm. 4 \\
        \midrule
        njNLL & -0.536 ± .04 & -0.540 ± .02 & -0.574 ± .03 & -0.526 ± .02 & -0.547 ± .03 \\
        \bottomrule
    \end{tabular}
\end{table}

%% file: appendix/query_dependent_weights.tex
\section{Quantifying the Effect of Query-Dependent SPN Weights}\label{app:consistency_exp}
\Cref{app:necessity} shows that query-independent effective mixture weights are necessary for marginalization consistency. To quantify how strongly a violation shows on real data, we construct an inconsistent variant of \model{} in which the SPN weights depend on the query $\mathcal{Q}$, thereby violating Property~3 of \Cref{l:lemma1}. In this variant, the learnable channel-wise queries $\mathbf{E}^\text{cw}$ that produce the sum-layer weights $\mathbf{W}^{(c)}$ (\Cref{eq:WCfromZ}) attend over the (linearly projected) query embeddings $[\phi^\textsc{time}(t^\text{qry}_n) \,\|\, \tilde{\mathbf{h}}_{c_n}]$ of all $N$ query points instead of over the channel states $\tilde{\mathbf{H}}$. The root weights and all other components remain unchanged.
We compare its marginalization consistency with that of the original \model{}, measured on samples.

On 100 test instances we build a synthetic query with a single query point in every channel, whose timestamps are drawn uniformly from the range of the observed timestamps of the instance; only the observations $\mathcal{X}$ are taken from the data. We then pick two distinct channels $(a,b)$ at random and draw $n = 1000$ samples of their bivariate marginal in two ways: once from the full all-channel query and once from a query containing only $a$ and $b$. A marginalization consistent model must yield the same bivariate distribution in both cases. We measure the discrepancy by the 2-Wasserstein distance between the two sample sets $\{x_i\}_{i=1}^n$ and $\{y_i\}_{i=1}^n$:
\begin{equation}
    W_2 = {\left( \min_{\sigma \in S_n} \frac{1}{n} \sum_{i=1}^{n} \lVert x_i - y_{\sigma(i)} \rVert_2^{2} \right)}^{1/2}.
\end{equation}
Since 1000 samples never reproduce a distribution exactly, we report the ratio of the mean discrepancy to a noise floor, the mean $W_2$ distance between two sets of 1000 samples of the same two-channel query drawn with different seeds. Values near $1$ indicate consistency.
\Cref{tab:consistency_exp} shows that the ratio of \model{} is close to $1$ on all datasets, whereas the query-dependent variant exceeds the noise floor by a factor of up to $5.5$. The ratio of $0.80$ on PhysioNet'12 is explained by the high variance of the sampling error on this dataset ($W_2$ between marginalized and two-channel query: $0.075 \pm 0.13$; $W_2$ between resampled two-channel queries: $0.093 \pm 0.32$). Hence, the inconsistency induced by query-dependent SPN weights is clearly measurable on real data.
\begin{table}[!ht]
    \centering
    \caption{Ratio of the $W_2$ discrepancy between the marginalized and the directly queried bivariate distribution to the resampling noise floor. Values near $1$ indicate marginalization consistency.}\label{tab:consistency_exp}
    \small
    \begin{tabular}{lcccc}
        \toprule
        Model & USHCN & PhysioNet'12 & MIMIC-III & MIMIC-IV \\
        \midrule
        \model{} & \textbf{0.99} & \textbf{0.80} & \textbf{1.00} & \textbf{1.01} \\
        \model{} (query-dep.\ weights) & 1.25 & 4.72 & 5.54 & 4.99 \\
        \bottomrule
    \end{tabular}
\end{table}

%% file: appendix/k_sensitivity.tex
\section{Sensitivity Analysis for the Number of Components}\label{app:k_sensitivity}
We evaluated the sensitivity of \model{} to the hyperparameter $K$, representing the number of leaf nodes per channel (mixture components). 
As shown in \Cref{fig:k_sensitivity}, the model exhibits moderate robustness to variations in $K$. 
Our results suggest that selecting $K$ from the set $\{2, 3, 4\}$ is sufficient, as the optimal configuration was found to be $K=2$ or $K=3$ for the evaluated task (36--12).
\input{tikz_images/k_sensitivity.tex}

%% file: tikz_images/k_sensitivity.tex
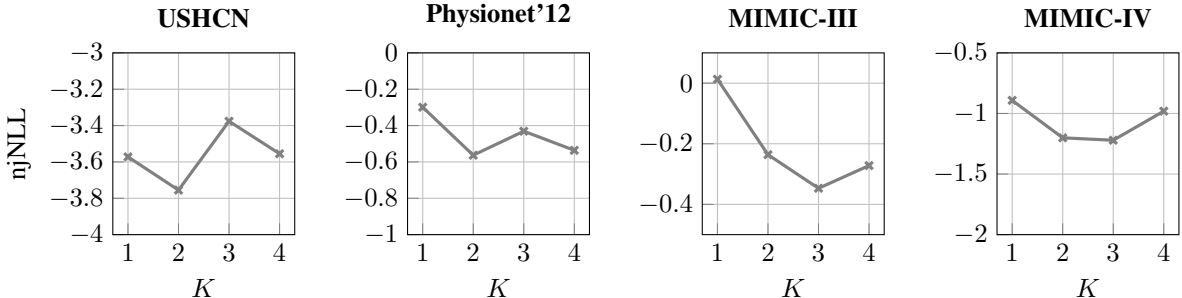
\begin{figure}[!ht]
    \centering
    \resizebox{\textwidth}{!}{%
    \begin{tikzpicture}
        \begin{groupplot}[
                group style={
                    group size=4 by 1,  
                    horizontal sep=1.5cm, 
                    vertical sep=1cm
                },
                width=4cm, height=4cm, 
                xlabel={$K$},
                grid=major,
                xtick={1,2,3,4},
                tick pos=left,
                yticklabel style={
                    /pgf/number format/fixed,
                    /pgf/number format/precision=3
                }
            ]

            \nextgroupplot[title=\textbf{USHCN}, ylabel={njNLL},ymin=-4, ymax=-3]
            \addplot[color=gray,mark=x,very thick] coordinates {
                (1, -3.572)
                (2, -3.755)
                (3, -3.376)
                (4, -3.555)
            };

            \nextgroupplot[title=\textbf{Physionet'12},ymin=-1, ymax=0]
            \addplot[color=gray,mark=x,very thick] coordinates {
                (1, -0.299)
                (2, -0.563)
                (3, -0.431)
                (4, -0.536)
            };

            \nextgroupplot[title=\textbf{MIMIC-III},,ymin=-.5, ymax=0.1]
            \addplot[color=gray,mark=x,very thick] coordinates {
                (1, 0.013)
                (2, -0.236)
                (3, -0.347)
                (4, -0.272)
            };

            \nextgroupplot[title=\textbf{MIMIC-IV}, ,ymin=-2, ymax=-0.5]
            \addplot[color=gray,mark=x,very thick] coordinates {
                (1, -0.892)
                (2, -1.201)
                (3, -1.221)
                (4, -0.981)
            };

        \end{groupplot}
    \end{tikzpicture}%
    }
    \caption{Sensitivity Analysis on the 36-12 task for the number of circuit components $K$.}\label{fig:k_sensitivity}
\end{figure}

%% file: appendix/scaling.tex
\section{Memory Consumption and Epoch Times}\label{app:effi}
We conduct an experiment to compare \model{}' memory consumption during training and epoch times with those of ProFITi and MOSES.\@ 
To do this, we use toy datasets based on Brownian motion, in which we can vary the number of channels $C$ and the number of time steps $L$. The models are tasked with predicting the second half of time steps based on the first half, similar to the 24--24 task in our main experiment (\Cref{sec:exp}). To simulate an IMTS, we randomly drop 10\% of all observations and forecasting queries. Hence, the expected number of query points $N$ is $\frac{0.9LC}{2}$. 
\input{figures/effi.tex}

For MOSES, we set the number of mixture components to 3, the hidden dimension to 128, and the number of attention heads to 4. ProFITi is applied using 7 flow layers and a latent dimension of 64. For \model{}, we set the number of components $K$ to 3, the number of flow layers to 2, and the hidden dimension to 32.
We run these experiments on a single NVIDIA A40 with a batch size of 64.
\Cref{fig:effi} presents the results. In this toy setup, where the channels all have a similar number of queries, \model{} is significantly faster than MOSES and ProFITi, while using less memory to compute the gradients and update the weights.        

As stated in \Cref{sec:complexity} \model{} reduces the complexity compared to a model that computes full covariance over all query points by factor $C^2$ under the assumption that the forecasting targets are distributed evenly among channels. 

However, in real world IMTS datasets this is typically not the case and some channels are observed more frequently than others. At least in our implementation the speed-up that our model generates by grouping the forecasting targets by channels decreases, when forecasting queries are concentrating on single channels. As an example we show the distribution of queries over channels for MIMIC-III in the 12--36 setting in \Cref{fig:mimic-distribtion}. For such uneven IMTS \model{} is actually slower than MOSES and ProFITi as we show in \Cref{tab:mimiciiiet}.

\begin{figure}[!htp]
    \centering
    \begin{minipage}[b]{0.34\textwidth}
        \centering\small\setlength{\tabcolsep}{4pt}
        \makeatletter\def\@captype{table}\makeatother 
        \caption{Epoch times on MIMIC-III, run on a single A100 GPU.}\label{tab:mimiciiiet}
        \begin{tabular}{@{}lrrr@{}}
            \toprule
            Model & 36--12 & 24--24 & 12--36 \\
            \midrule
            ProFITi &  13s & 14s  &  15s  \\
            MOSES   &  13s & 12s  &  13s  \\
            \model{} & 27s & 24s  &  40s  \\
            \bottomrule
        \end{tabular}
        \vspace{1.5em}
    \end{minipage}\hfill
    \begin{minipage}[b]{0.62\textwidth}
        \centering
        \includegraphics[width=\textwidth]{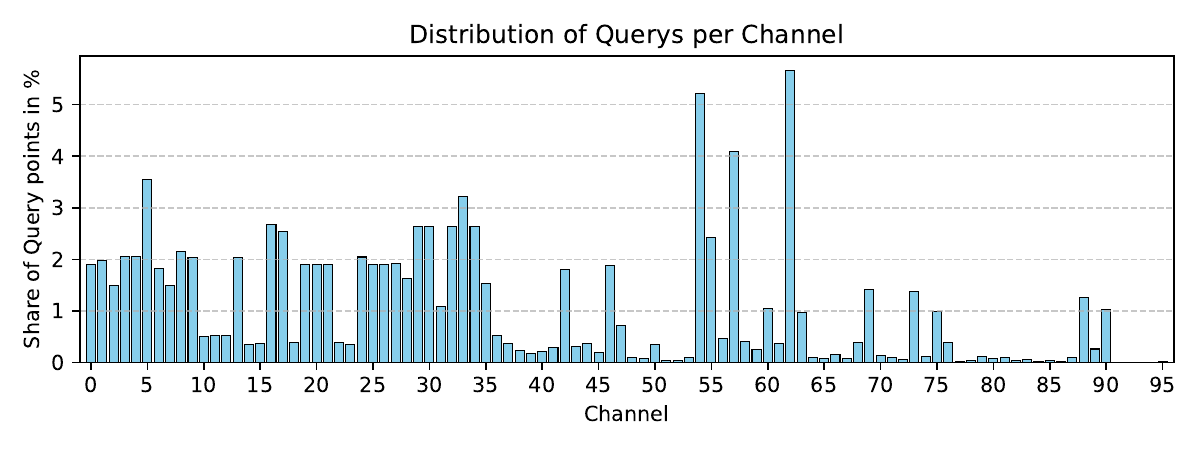}
        \caption{Query distribution for MIMIC-III 12--36.}\label{fig:mimic-distribtion}
    \end{minipage}
\end{figure}

%% file: figures/effi.tex
\begin{figure}[!htp]
    \centering
    \begin{tikzpicture}
        \begin{groupplot}[
            group style={
                group size=2 by 2,
                horizontal sep=1.5cm,
                vertical sep=0.5cm,
                x descriptions at=edge bottom,
                y descriptions at=edge left,
            },
            width=6cm,
            height=3.3cm,
            grid=major,
            legend style={
                at={(0.5,-0.35)}, 
                anchor=north, 
                legend columns=-1,
                /tikz/every even column/.append style={column sep=0.5cm}
            },
            cycle list name=color list
        ]

        \nextgroupplot[
            title={\textbf{Scaling with Channels}},
            ylabel={Epoch Time (s)}
        ]
        \addplot[color=moses_color, mark=square*] coordinates {
            (10, 11.2) (20, 25.2) (40, 73.3)
        };
        \addplot[color=circuits_color, mark=triangle*] coordinates {
            (10, 4.4) (20, 8.4) (40, 23.4)
        };
        \addplot[color=profiti_color, mark=*] coordinates {
            (10, 8.8) (20, 23.4) (40, 76.8)
        };

        \nextgroupplot[
            title={\textbf{Scaling with Time Steps}},
            legend to name=CommonLegend 
        ]
        \addplot[color=moses_color, mark=square*] coordinates {
            (25, 11.2) (50, 25.2) (100, 74.3)
        };
        \addlegendentry{MOSES}
        \addplot[color=circuits_color, mark=triangle*] coordinates {
            (25, 4.9) (50, 8.4) (100, 16.9)
        };
        \addlegendentry{CircuITS}
        \addplot[color=profiti_color, mark=*] coordinates {
            (25, 8.6) (50, 23.4) (100, 78.55)
        };
        \addlegendentry{ProFITi}

        \nextgroupplot[
            ylabel={Memory (GB)},
            xlabel={Channels ($C$)}
        ]
        \addplot[color=moses_color, mark=square*] coordinates {
            (10, 2.05) (20, 5.4) (40, 16.1)
        };
        \addplot[color=circuits_color, mark=triangle*] coordinates {
            (10, 0.6) (20, 1.9) (40, 6.3)
        };
        \addplot[color=profiti_color, mark=*] coordinates {
            (10, 2.1) (20, 6.2) (40, 19.9)
        };

        \nextgroupplot[
            xlabel={Time Steps ($L$)}
        ]
        \addplot[color=moses_color, mark=square*] coordinates {
            (25, 2.1) (50, 5.4) (100, 16.1)
        };
        \addplot[color=circuits_color, mark=triangle*] coordinates {
            (25, 0.9) (50, 1.9) (100, 4.3)
        };
        \addplot[color=profiti_color, mark=*] coordinates {
            (25, 2.1) (50, 6.2) (100, 20.2)
        };

        \end{groupplot}
        
        \node[yshift=-1.2cm] at ($(group c1r2.south)!0.5!(group c2r2.south)$) {\ref{CommonLegend}};

    \end{tikzpicture}
    \caption{Comparing memory consumption and epoch times during training on Toy datasets.}\label{fig:effi}
\end{figure}

%% file: appendix/broader_impacts.tex
\section{Broader Impacts}\label{app:broader_impacts}
\model{} improves the calibration and consistency of probabilistic forecasts for irregular multivariate time series, which can support better-informed decision-making in domains such as climate monitoring and clinical care --- two of the application areas represented in our benchmarks, alongside human motion tracking. Reliable uncertainty quantification is particularly valuable in healthcare settings, where overconfident point predictions can mislead clinicians. At the same time, more accurate-looking probabilistic forecasts can also induce automation bias if practitioners trust them beyond what calibration warrants, and the model's guarantees concern marginalization consistency rather than the correctness of the underlying data or domain shift at deployment time. We therefore view \model{} as a tool to be used alongside, not in place of, domain expertise. We do not release new datasets and use only existing public benchmarks (USHCN, PhysioNet 2012, MIMIC-III, MIMIC-IV, Human Activity), with the clinical datasets de-identified by their creators, under their respective licenses, so we do not see additional privacy or dual-use concerns introduced by this work.

%% file: checklist.tex
\section*{NeurIPS Paper Checklist}
\begin{enumerate}

\item {\bf Claims}
    \item[] Question: Do the main claims made in the abstract and introduction accurately reflect the paper's contributions and scope?
    \item[] Answer: \answerYes{} 
    \item[] Justification: The contributions stated in the introduction (a marginalization-consistent SPN architecture for IMTS, a query encoder, and superior joint density estimation on five datasets) are directly supported by \Cref{app:gc_proof,app:proof}, the architecture in \Cref{sec:exp}, and the results in \Cref{tab:njnll}.
    \item[] Guidelines:
    \begin{itemize}
        \item The answer \answerNA{} means that the abstract and introduction do not include the claims made in the paper.
        \item The abstract and/or introduction should clearly state the claims made, including the contributions made in the paper and important assumptions and limitations. A \answerNo{} or \answerNA{} answer to this question will not be perceived well by the reviewers. 
        \item The claims made should match theoretical and experimental results, and reflect how much the results can be expected to generalize to other settings. 
        \item It is fine to include aspirational goals as motivation as long as it is clear that these goals are not attained by the paper. 
    \end{itemize}

\item {\bf Limitations}
    \item[] Question: Does the paper discuss the limitations of the work performed by the authors?
    \item[] Answer: \answerYes{} 
    \item[] Justification: A dedicated Limitations section discusses channel-order dependence of the sequential SPN construction, the inability to compactly represent smooth inter-channel dependencies, and that the $\mathcal{O}(N^3/C^2)$ complexity assumes evenly distributed queries (with longer epoch times when queries concentrate on few channels).
    \item[] Guidelines:
    \begin{itemize}
        \item The answer \answerNA{} means that the paper has no limitation while the answer \answerNo{} means that the paper has limitations, but those are not discussed in the paper. 
        \item The authors are encouraged to create a separate ``Limitations'' section in their paper.
        \item The paper should point out any strong assumptions and how robust the results are to violations of these assumptions (e.g., independence assumptions, noiseless settings, model well-specification, asymptotic approximations only holding locally). The authors should reflect on how these assumptions might be violated in practice and what the implications would be.
        \item The authors should reflect on the scope of the claims made, e.g., if the approach was only tested on a few datasets or with a few runs. In general, empirical results often depend on implicit assumptions, which should be articulated.
        \item The authors should reflect on the factors that influence the performance of the approach. For example, a facial recognition algorithm may perform poorly when image resolution is low or images are taken in low lighting. Or a speech-to-text system might not be used reliably to provide closed captions for online lectures because it fails to handle technical jargon.
        \item The authors should discuss the computational efficiency of the proposed algorithms and how they scale with dataset size.
        \item If applicable, the authors should discuss possible limitations of their approach to address problems of privacy and fairness.
        \item While the authors might fear that complete honesty about limitations might be used by reviewers as grounds for rejection, a worse outcome might be that reviewers discover limitations that aren't acknowledged in the paper. The authors should use their best judgment and recognize that individual actions in favor of transparency play an important role in developing norms that preserve the integrity of the community. Reviewers will be specifically instructed to not penalize honesty concerning limitations.
    \end{itemize}

\item {\bf Theory assumptions and proofs}
    \item[] Question: For each theoretical result, does the paper provide the full set of assumptions and a complete (and correct) proof?
    \item[] Answer: \answerYes{} 
    \item[] Justification: The marginalization-consistency results (Lemma in \Cref{app:proof} and Proposition in \Cref{app:gc_proof}) are formally stated with their assumptions and accompanied by complete proofs in the appendix.
    \item[] Guidelines:
    \begin{itemize}
        \item The answer \answerNA{} means that the paper does not include theoretical results. 
        \item All the theorems, formulas, and proofs in the paper should be numbered and cross-referenced.
        \item All assumptions should be clearly stated or referenced in the statement of any theorems.
        \item The proofs can either appear in the main paper or the supplemental material, but if they appear in the supplemental material, the authors are encouraged to provide a short proof sketch to provide intuition. 
        \item Inversely, any informal proof provided in the core of the paper should be complemented by formal proofs provided in appendix or supplemental material.
        \item Theorems and Lemmas that the proof relies upon should be properly referenced. 
    \end{itemize}

    \item {\bf Experimental result reproducibility}
    \item[] Question: Does the paper fully disclose all the information needed to reproduce the main experimental results of the paper to the extent that it affects the main claims and/or conclusions of the paper (regardless of whether the code and data are provided or not)?
    \item[] Answer: \answerYes{} 
    \item[] Justification: The architecture is fully specified in \Cref{sec:exp} and the appendix, datasets and the preprocessing/splitting protocol of~\citet{Yalavarthi2025.Reliable} are described in \Cref{app:datasets}, hyperparameters and search spaces in \Cref{app:hyper}, and the code is publicly available (linked from the introduction).
    \item[] Guidelines:
    \begin{itemize}
        \item The answer \answerNA{} means that the paper does not include experiments.
        \item If the paper includes experiments, a \answerNo{} answer to this question will not be perceived well by the reviewers: Making the paper reproducible is important, regardless of whether the code and data are provided or not.
        \item If the contribution is a dataset and\slash{}or model, the authors should describe the steps taken to make their results reproducible or verifiable.
        \item Depending on the contribution, reproducibility can be accomplished in various ways. For example, if the contribution is a novel architecture, describing the architecture fully might suffice, or if the contribution is a specific model and empirical evaluation, it may be necessary to either make it possible for others to replicate the model with the same dataset, or provide access to the model. In general,\ releasing code and data is often one good way to accomplish this, but reproducibility can also be provided via detailed instructions for how to replicate the results, access to a hosted model (e.g., in the case of a large language model), releasing of a model checkpoint, or other means that are appropriate to the research performed.
        \item While NeurIPS does not require releasing code, the conference does require all submissions to provide some reasonable avenue for reproducibility, which may depend on the nature of the contribution. For example
        \begin{enumerate}
            \item If the contribution is primarily a new algorithm, the paper should make it clear how to reproduce that algorithm.
            \item If the contribution is primarily a new model architecture, the paper should describe the architecture clearly and fully.
            \item If the contribution is a new model (e.g., a large language model), then there should either be a way to access this model for reproducing the results or a way to reproduce the model (e.g., with an open-source dataset or instructions for how to construct the dataset).
            \item We recognize that reproducibility may be tricky in some cases, in which case authors are welcome to describe the particular way they provide for reproducibility. In the case of closed-source models, it may be that access to the model is limited in some way (e.g., to registered users), but it should be possible for other researchers to have some path to reproducing or verifying the results.
        \end{enumerate}
    \end{itemize}

\item {\bf Open access to data and code}
    \item[] Question: Does the paper provide open access to the data and code, with sufficient instructions to faithfully reproduce the main experimental results, as described in supplemental material?
    \item[] Answer: \answerYes{} 
    \item[] Justification: We provide a public code repository (linked in the footnote of the introduction) containing the model implementation and training scripts. Datasets (USHCN, PhysioNet'12, MIMIC-III, MIMIC-IV, Human Activity) are public benchmarks; for the first four we follow the preprocessing and splitting protocol of~\citet{Yalavarthi2025.Reliable}, and we describe the preprocessing of all datasets in \Cref{app:datasets}.
    \item[] Guidelines:
    \begin{itemize}
        \item The answer \answerNA{} means that paper does not include experiments requiring code.
        \item Please see the NeurIPS code and data submission guidelines (\url{https://neurips.cc/public/guides/CodeSubmissionPolicy}) for more details.
        \item While we encourage the release of code and data, we understand that this might not be possible, so \answerNo{} is an acceptable answer. Papers cannot be rejected simply for not including code, unless this is central to the contribution (e.g., for a new open-source benchmark).
        \item The instructions should contain the exact command and environment needed to run to reproduce the results. See the NeurIPS code and data submission guidelines (\url{https://neurips.cc/public/guides/CodeSubmissionPolicy}) for more details.
        \item The authors should provide instructions on data access and preparation, including how to access the raw data, preprocessed data, intermediate data, and generated data, etc.
        \item The authors should provide scripts to reproduce all experimental results for the new proposed method and baselines. If only a subset of experiments are reproducible, they should state which ones are omitted from the script and why.
        \item At submission time, to preserve anonymity, the authors should release anonymized versions (if applicable).
        \item Providing as much information as possible in supplemental material (appended to the paper) is recommended, but including URLs to data and code is permitted.
    \end{itemize}

\item {\bf Experimental setting/details}
    \item[] Question: Does the paper specify all the training and test details (e.g., data splits, hyperparameters, how they were chosen, type of optimizer) necessary to understand the results?
    \item[] Answer: \answerYes{} 
    \item[] Justification: \Cref{app:hyper} specifies the optimizer (AdamW), learning rate, weight decay, batch size, scheduler, early-stopping criterion, and the random hyperparameter search spaces for \model{} and all baselines; data splits follow the protocol of~\citet{Yalavarthi2025.Reliable} (\Cref{app:datasets}).
    \item[] Guidelines:
    \begin{itemize}
        \item The answer \answerNA{} means that the paper does not include experiments.
        \item The experimental setting should be presented in the core of the paper to a level of detail that is necessary to appreciate the results and make sense of them.
        \item The full details can be provided either with the code, in appendix, or as supplemental material.
    \end{itemize}

\item {\bf Experiment statistical significance}
    \item[] Question: Does the paper report error bars suitably and correctly defined or other appropriate information about the statistical significance of the experiments?
    \item[] Answer: \answerYes{} 
    \item[] Justification: All result tables report mean $\pm$ standard deviation across five paired train/validation/test splits, with the same five splits used for every method (\Cref{sec:exp}, \Cref{tab:njnll}).
    \item[] Guidelines:
    \begin{itemize}
        \item The answer \answerNA{} means that the paper does not include experiments.
        \item The authors should answer \answerYes{} if the results are accompanied by error bars, confidence intervals, or statistical significance tests, at least for the experiments that support the main claims of the paper.
        \item The factors of variability that the error bars are capturing should be clearly stated (for example, train/test split, initialization, random drawing of some parameter, or overall run with given experimental conditions).
        \item The method for calculating the error bars should be explained (closed form formula, call to a library function, bootstrap, etc.)
        \item The assumptions made should be given (e.g., Normally distributed errors).
        \item It should be clear whether the error bar is the standard deviation or the standard error of the mean.
        \item It is OK to report 1-sigma error bars, but one should state it. The authors should preferably report a 2-sigma error bar than state that they have a 96\% CI, if the hypothesis of Normality of errors is not verified.
        \item For asymmetric distributions, the authors should be careful not to show in tables or figures symmetric error bars that would yield results that are out of range (e.g., negative error rates).
        \item If error bars are reported in tables or plots, the authors should explain in the text how they were calculated and reference the corresponding figures or tables in the text.
    \end{itemize}

\item {\bf Experiments compute resources}
    \item[] Question: For each experiment, does the paper provide sufficient information on the computer resources (type of compute workers, memory, time of execution) needed to reproduce the experiments?
    \item[] Answer: \answerYes{} 
    \item[] Justification: The compute hardware (NVIDIA A100/A40 GPUs) is reported in \Cref{app:effi}, and per-epoch runtimes for \model{} and baselines are given in \Cref{tab:mimiciiiet}.
    \item[] Guidelines:
    \begin{itemize}
        \item The answer \answerNA{} means that the paper does not include experiments.
        \item The paper should indicate the type of compute workers CPU or GPU, internal cluster, or cloud provider, including relevant memory and storage.
        \item The paper should provide the amount of compute required for each of the individual experimental runs as well as estimate the total compute. 
        \item The paper should disclose whether the full research project required more compute than the experiments reported in the paper (e.g., preliminary or failed experiments that didn't make it into the paper). 
    \end{itemize}
    
\item {\bf Code of ethics}
    \item[] Question: Does the research conducted in the paper conform, in every respect, with the NeurIPS Code of Ethics \url{https://neurips.cc/public/EthicsGuidelines}?
    \item[] Answer: \answerYes{} 
    \item[] Justification: The work uses only existing, de-identified public benchmarks under their respective licenses, releases no new data or high-risk models; we have reviewed and conform to the NeurIPS Code of Ethics.
    \item[] Guidelines:
    \begin{itemize}
        \item The answer \answerNA{} means that the authors have not reviewed the NeurIPS Code of Ethics.
        \item If the authors answer \answerNo, they should explain the special circumstances that require a deviation from the Code of Ethics.
        \item The authors should make sure to preserve anonymity (e.g., if there is a special consideration due to laws or regulations in their jurisdiction).
    \end{itemize}

\item {\bf Broader impacts}
    \item[] Question: Does the paper discuss both potential positive societal impacts and negative societal impacts of the work performed?
    \item[] Answer: \answerYes{} 
    \item[] Justification: \Cref{app:broader_impacts} discusses positive impacts (better-calibrated probabilistic forecasts for clinical and climate applications) and negative impacts (automation bias, the gap between marginalization consistency and downstream correctness under domain shift).
    \item[] Guidelines:
    \begin{itemize}
        \item The answer \answerNA{} means that there is no societal impact of the work performed.
        \item If the authors answer \answerNA{} or \answerNo, they should explain why their work has no societal impact or why the paper does not address societal impact.
        \item Examples of negative societal impacts include potential malicious or unintended uses (e.g., disinformation, generating fake profiles, surveillance), fairness considerations (e.g., deployment of technologies that could make decisions that unfairly impact specific groups), privacy considerations, and security considerations.
        \item The conference expects that many papers will be foundational research and not tied to particular applications, let alone deployments. However, if there is a direct path to any negative applications, the authors should point it out. For example, it is legitimate to point out that an improvement in the quality of generative models could be used to generate Deepfakes for disinformation. On the other hand, it is not needed to point out that a generic algorithm for optimizing neural networks could enable people to train models that generate Deepfakes faster.
        \item The authors should consider possible harms that could arise when the technology is being used as intended and functioning correctly, harms that could arise when the technology is being used as intended but gives incorrect results, and harms following from (intentional or unintentional) misuse of the technology.
        \item If there are negative societal impacts, the authors could also discuss possible mitigation strategies (e.g., gated release of models, providing defenses in addition to attacks, mechanisms for monitoring misuse, mechanisms to monitor how a system learns from feedback over time, improving the efficiency and accessibility of ML).
    \end{itemize}
    
\item {\bf Safeguards}
    \item[] Question: Does the paper describe safeguards that have been put in place for responsible release of data or models that have a high risk for misuse (e.g., pre-trained language models, image generators, or scraped datasets)?
    \item[] Answer: \answerNA{} 
    \item[] Justification: We release neither scraped datasets nor generative models with significant misuse potential; the released artifact is a forecasting model trained on existing public benchmarks.
    \item[] Guidelines:
    \begin{itemize}
        \item The answer \answerNA{} means that the paper poses no such risks.
        \item Released models that have a high risk for misuse or dual-use should be released with necessary safeguards to allow for controlled use of the model, for example by requiring that users adhere to usage guidelines or restrictions to access the model or implementing safety filters. 
        \item Datasets that have been scraped from the Internet could pose safety risks. The authors should describe how they avoided releasing unsafe images.
        \item We recognize that providing effective safeguards is challenging, and many papers do not require this, but we encourage authors to take this into account and make a best faith effort.
    \end{itemize}

\item {\bf Licenses for existing assets}
    \item[] Question: Are the creators or original owners of assets (e.g., code, data, models), used in the paper, properly credited and are the license and terms of use explicitly mentioned and properly respected?
    \item[] Answer: \answerYes{} 
    \item[] Justification: All baselines (ProFITi, MOSES, NeuralFlows, Tripletformer) and datasets (USHCN, PhysioNet'12, MIMIC-III, MIMIC-IV, Human Activity) are cited to their original sources in \Cref{app:datasets} and the main text, and used in accordance with their respective licenses/data-use agreements; Human Activity is distributed by the UCI Machine Learning Repository under CC BY 4.0.
    \item[] Guidelines:
    \begin{itemize}
        \item The answer \answerNA{} means that the paper does not use existing assets.
        \item The authors should cite the original paper that produced the code package or dataset.
        \item The authors should state which version of the asset is used and, if possible, include a URL\@.
        \item The name of the license (e.g., CC-BY 4.0) should be included for each asset.
        \item For scraped data from a particular source (e.g., website), the copyright and terms of service of that source should be provided.
        \item If assets are released, the license, copyright information, and terms of use in the package should be provided. For popular datasets, \url{paperswithcode.com/datasets} has curated licenses for some datasets. Their licensing guide can help determine the license of a dataset.
        \item For existing datasets that are re-packaged, both the original license and the license of the derived asset (if it has changed) should be provided.
        \item If this information is not available online, the authors are encouraged to reach out to the asset's creators.
    \end{itemize}

\item {\bf New assets}
    \item[] Question: Are new assets introduced in the paper well documented and is the documentation provided alongside the assets?
    \item[] Answer: \answerYes{} 
    \item[] Justification: The only new asset is the \model{} code, released publicly on GitHub (link in the introduction footnote) with documentation of training and evaluation; no new datasets or models are released.
    \item[] Guidelines:
    \begin{itemize}
        \item The answer \answerNA{} means that the paper does not release new assets.
        \item Researchers should communicate the details of the dataset\slash{}code\slash{}model as part of their submissions via structured templates. This includes details about training, license, limitations, etc.
        \item The paper should discuss whether and how consent was obtained from people whose asset is used.
        \item At submission time, remember to anonymize your assets (if applicable). You can either create an anonymized URL or include an anonymized zip file.
    \end{itemize}

\item {\bf Crowdsourcing and research with human subjects}
    \item[] Question: For crowdsourcing experiments and research with human subjects, does the paper include the full text of instructions given to participants and screenshots, if applicable, as well as details about compensation (if any)?
    \item[] Answer: \answerNA{} 
    \item[] Justification: The paper does not involve crowdsourcing or research with human subjects; we use only existing, de-identified public benchmark datasets.
    \item[] Guidelines:
    \begin{itemize}
        \item The answer \answerNA{} means that the paper does not involve crowdsourcing nor research with human subjects.
        \item Including this information in the supplemental material is fine, but if the main contribution of the paper involves human subjects, then as much detail as possible should be included in the main paper. 
        \item According to the NeurIPS Code of Ethics, workers involved in data collection, curation, or other labor should be paid at least the minimum wage in the country of the data collector. 
    \end{itemize}

\item {\bf Institutional review board (IRB) approvals or equivalent for research with human subjects}
    \item[] Question: Does the paper describe potential risks incurred by study participants, whether such risks were disclosed to the subjects, and whether Institutional Review Board (IRB) approvals (or an equivalent approval/review based on the requirements of your country or institution) were obtained?
    \item[] Answer: \answerNA{} 
    \item[] Justification: The paper does not conduct new research with human subjects; the MIMIC-III/IV and PhysioNet datasets used were collected and de-identified by their original creators under their own IRB approvals.
    \item[] Guidelines:
    \begin{itemize}
        \item The answer \answerNA{} means that the paper does not involve crowdsourcing nor research with human subjects.
        \item Depending on the country in which research is conducted, IRB approval (or equivalent) may be required for any human subjects research. If you obtained IRB approval, you should clearly state this in the paper. 
        \item We recognize that the procedures for this may vary significantly between institutions and locations, and we expect authors to adhere to the NeurIPS Code of Ethics and the guidelines for their institution. 
        \item For initial submissions, do not include any information that would break anonymity (if applicable), such as the institution conducting the review.
    \end{itemize}

\item {\bf Declaration of LLM usage}
    \item[] Question: Does the paper describe the usage of LLMs if it is an important, original, or non-standard component of the core methods in this research? Note that if the LLM is used only for writing, editing, or formatting purposes and does \emph{not} impact the core methodology, scientific rigor, or originality of the research, declaration is not required.
    \item[] Answer: \answerNA{} 
    \item[] Justification: LLMs are not part of the core methodology of this paper; any use was limited to writing/editing assistance, which does not require declaration under the policy.
    \item[] Guidelines:
    \begin{itemize}
        \item The answer \answerNA{} means that the core method development in this research does not involve LLMs as any important, original, or non-standard components.
        \item Please refer to our LLM policy in the NeurIPS handbook for what should or should not be described.
    \end{itemize}

\end{enumerate}